\documentclass[11pt,twoside]{article}
\usepackage{url}
\usepackage{graphicx}
\usepackage{makeidx}
\usepackage{multicol}
\usepackage{amssymb}
\usepackage{subfigure}
\usepackage{algorithm}
\usepackage{algorithmic}
\usepackage{jair}
\usepackage{theapa}

\input epsf 
\jairheading{28}{2007}{299-348}{07/06}{03/07}
\ShortHeadings{Mapping Calendar Expressions to Minimal Periodic Sets}
{Bettini, Mascetti \& Wang}
\firstpageno{299}

\graphicspath{{Figure/}}

\makeindex

\newtheorem{theo}{Theorem}

\newtheorem{prop}{Proposition}

\newtheorem{definition}{Definition}
\newtheorem{example}{Example}
\newenvironment{proof}{\noindent{\em Proof.}}{\hfill $\Box$\vspace{1ex}}

\newcommand{\up}[1]{\lceil #1\rceil}
\newcommand{\down}[1]{\lfloor #1\rfloor}

\newcommand{\groups}{\ \underline{\triangleleft}\ }
\newcommand{\finer}{\preceq}

\newcommand{\gpinto}{\ \bar{{\underline{\triangleleft}\ }}}

\newcommand{\Lset}[1]{{\cal L}_{#1}} 
\newcommand{\Lbar}[1]{\overline{{\cal L}}_{#1}} 
\newcommand{\Lcap}[1]{\hat{\cal L}_{#1}} 
\newcommand{\Ltil}[1]{\widetilde{\cal L}_{#1}}

\begin{document}
\title{Supporting Temporal Reasoning by Mapping \\ Calendar
Expressions to Minimal Periodic Sets}

\author{\name Claudio Bettini \email bettini@dico.unimi.it\\
       \name Sergio Mascetti \email mascetti@dico.unimi.it \\
       \addr Dipartimento di Informatica e Comunicazione, Universit\`a di Milano\\
       Via Comelico, 39, 20135, Milan, Italy
       \AND
       \name X.~Sean Wang \email Sean.Wang@uvm.edu \\
       \addr Department of Computer Science, University of Vermont\\
       33 Colchester Avenue, Burlington, VT, 05405 USA}

\date{}
\maketitle
                     
\begin{abstract}
In the recent years several research efforts have focused on the
concept of time granularity and its applications. A first stream of
research investigated the mathematical models behind the notion of granularity and
the algorithms to manage temporal data based on those models. A
second stream of research investigated symbolic formalisms providing
a set of algebraic operators to define granularities in a compact
and compositional way. However, only very limited manipulation
algorithms have been proposed to operate directly on the algebraic
representation
making it unsuitable to use the symbolic formalisms in applications
that need manipulation of granularities. 

This paper aims at filling the gap between 
the
results from these two streams of research, by providing an efficient
conversion from the algebraic representation to the equivalent
low-level representation based on the mathematical models.  In
addition, the conversion returns a minimal representation in terms of
period length.  Our results have a major practical impact: users can
more easily define arbitrary granularities in terms of algebraic
operators, and then access granularity reasoning
and other services operating efficiently on the equivalent, minimal low-level
representation.  As an example, we illustrate the application to
temporal constraint reasoning with multiple granularities.

{}From a technical point of view, we propose an hybrid algorithm that
interleaves the conversion of calendar subexpressions into periodical
sets with the minimization of the period length. The algorithm returns
set-based granularity representations having minimal period length,
which is the most relevant parameter for the performance of the
considered reasoning services.  Extensive experimental work supports
the techniques used in the algorithm, and shows the efficiency and
effectiveness of the algorithm.
\end{abstract}

\section{Introduction}
\label{ch:intro}

According to a 2006 research by Oxford University Press, the word
\emph{time} has been found to be the most common noun in the English
language, considering diverse sources on the 
Internet
including
newspapers, journals, fictions and weblogs. What is somehow surprising is
that among the 25 most common nouns we find time granularities
like \emph{day}, \emph{week}, \emph{month} and \emph{year}.  We are pretty sure that many other
time granularities like \emph{business day}, \emph{quarter}, \emph{semester}, etc.
would be found to be quite 
frequently used
in natural 
languages.
However, the way computer
applications deal with these concepts is still very naive and mostly
hidden in program code and/or based on limited and sometimes imprecise
calendar support.  Temporal representation and reasoning has been for
a long time an AI research topic aimed at providing a formal framework
for common sense reasoning, natural language understanding, planning,
diagnosis and many other complex tasks involving time data management.  
Despite the many relevant contributions, time granularity representation and reasoning support
has very often been ignored or over-simplified. 
In the very active area of temporal constraint satisfaction, most
proposals implicitly assumed that adding support for granularity was a
trivial extension.  Only quite recently it was recognized that this is
not the case and specific techniques were proposed \cite{BWJ:aij02}.
Even the intuitively simple task of deciding whether a specific
instant is part of a time granularity 
can be tricky when arbitrary user-defined
granularities like e.g., \emph{banking days}, or \emph{academic semesters} are
considered.

Granularities and periodic patterns in terms of granularities are
playing a role even in emerging application areas like
inter-organizational workflows 
and personal information management (PIM).  For example, inter-organizational workflows
need to model and monitor constraints like: \emph{Event2 should occur no later than two
business days after the occurrence of Event1}.
In the context of PIM, current calendar applications, even on mobile
devices, allow the user to specify quite involved periodical patterns
for the recurrence of events. For example, it is possible to schedule
an event every last Saturday
of 
every two months.  The complexity of the
supported patterns has been increasing in the last years, and the
current simple interfaces are showing their limits.  They are
essentially based on a combination of recurrences based on one or two
granularities taken from a fixed set (days, weeks, months, and years).
We foresee the possibility for significant extensions of these
applications by specifying recurrences over user-defined
granularities. For example, the user may define (or upload from a
granularity library) the granularity corresponding to the academic
semester of the school he is teaching at, and set the date of the
finals as the last Monday of each semester. A bank 
may
want to define its 
\emph{banking days} granularity and some of the bank
policies may 
then
be formalized as recurrences in terms of that
granularity. Automatically generated appointments from these policies
may appear on the 
devices
of bank employees involved in
specific procedures.  We also foresee the need to show a user
preferred view of the calendar. With current standard applications the
user has a choice between a business-day limited view and a complete
view, but why not enabling a view based on the users's
\emph{consulting-days}, for example?  
A new perspective in the use of
mobile devices may also result from 
considering the time span in which activities are supposed to be executed
(expressed in arbitrary granularities), and having software agents on
board to alert about constraints that may be violated, even based on
contextual information like the user location or traffic conditions.
This scenario highlights three main requirements: a) a sufficiently
expressive formal model for time 
granularity,
b) a convenient way to
define new time granularities, and c) efficient reasoning tools over
time granularities.

Consider a). In
the last decade significant efforts have been made to provide
formal models for the notion of time granularity and to devise algorithms
to manage temporal data based on those models.
In addition to \emph{logical} approaches
\cite{Montanari-Thesis,CombiJLC04}, a framework based on periodic-set
representations has been extensively studied \cite{BJW:book}, and more
recently an approach based on strings and automata was introduced
\cite{Wijsen:AAAI2000,BresolinMP04}.  We are mostly interested in the
last two approaches because they support the effective computation of
basic operations on time granularities. In both cases the
representation of granularities can be considered as a
\emph{low-level} one, with a rather involved specification in terms of
the instants of the time domain.

Consider requirement b) above.
Users may have a hard time in defining granularities in formalisms based on 
low-level representations, and to interpret the output of operations.
%
It is clearly unreasonable to ask users to specify granularities by
linear equations or other mathematical formalisms that operate
directly in terms of instants or of granules of a fixed time granularity. 
Hence, a second stream of research
investigated more \emph{high-level} symbolic formalisms providing a
set of algebraic operators to define granularities in a compact and
compositional way.  The efforts on this task started even before the
research on formal models for granularity \cite{Leban-et-al:86,Niezette:92} and
continued as a parallel stream of research \cite{BD:amai00,NWJ:amai02,Ter:TKDE03,TauZaman-SPE07}.

Finally, let us consider requirement c) above. Several inferencing
operations have been defined on low-level representations, including
equivalence, inclusion between granules in different granularities,
and even complex inferencing services like constraint propagation
\cite{BWJ:aij02}.  Even for simple operations no general method is
available operating directly on the high level representation.
Indeed, in some cases, the proposed methods cannot exploit the
structure of the expression and require the enumeration of granules,
which may be very inefficient. This is the case, for example, of the
granule conversion methods presented by Ning e at. \citeyear{NWJ:amai02}. Moreover,
we are not aware of any method to perform other operations, such as
equivalence or intersection of sets of granules, directly in terms of
the high level representation.

The major goal of this paper is to provide a unique framework to
satisfy the requirements a), b), and c) identified above, by
adding to the existing results a smart and efficient technique to
convert granularity specifications from the high-level algebraic
formalism to the low-level one, for which many more reasoning tools
are available.
%
%
In particular, in this paper we focus on the conversion from the
high-level formalism called \emph{Calendar Algebra} \cite{NWJ:amai02}
to the low-level formalism based on periodical sets \cite{BJW:book,BWJ:aij02}.
Among the several proposals for the high-level (algebraic) specification of granularities,
the choice of Calendar Algebra has two main motivations: first, it allows the user to
express a large class of granularities; For a
comparison of the expressiveness of Calendar Algebra with other
formalisms see \cite{BJW:book}. Second, it provides the richest set of algebraic operations
that are designed to reflect the intuitive ways in which users define new granularities.
A discussion on the actual usability of this tool and on how it could be 
enhanced by a graphical user interface can be found in Section~\ref{subsec:CalDef}.
%
The choice of the low-level formalism based on periodic-sets 
also
has two main motivations: first, an efficient implementation of all
the basic operations already exists and has been extensively
experimented \cite{BMP-LNAI05}; second, it is the only one currently
supporting the complex operations on granularities needed for
constraint satisfaction, as it will be illustrated in more detail in Section~\ref{sec:gstp}.


%
%

%
%

The technical contribution of this paper is a hybrid algorithm
that interleaves the conversion of calendar subexpressions into
periodical sets with a step for period minimization. 
A central phase
of our conversion procedure is to derive, for each algebraic subexpression,
the periodicity of the output set.
This periodicity is used to build the periodical representation of the subexpression
that can be recursively used as operand of other expressions.
Given a calendar algebra expression, the
algorithm returns set-based granularity representations having
minimal period length. 
%
The period length is the most relevant
parameter for the performance both of basic operations on
granularities and of more specialized ones like the operations used by
the constraint satisfaction service.
Extensive experimental work
reported in this paper
validates the techniques used in the
algorithm, by showing, among other things, that
(1)
even large calendar
expressions can be efficiently converted, and 
(2)
less
precise
conversion formulas may lead to unacceptable computation time.
This latter property shows the importance of carefully and accurately
designed conversion formulas. Indeed, conversion formulas may seem
trivial if the length of periodicity is not a concern. In designing our
conversion formulas, we made an effort to reduce the period length of the
resulting granularity representation, and thus render the whole conversion process
computationally efficient.
%

%

%

%

In the next section we define granularities; several interesting
relationships among them are highlighted and the periodical set
representation is formalized.  In Section~\ref{ch:CalAlg} we define
Calendar Algebra and present its operations. In
Section~\ref{ch:CalAlg2PSet} we describe the conversion process: after
the definition of the three steps necessary for the conversion, for
each algebraic operation we present the formulas to perform each step.
In Section~\ref{sec:minimality} we discuss the period minimality
issue, and we report experimental results based on a full
implementation of the conversion algorithm and of its extension
ensuring minimality. In Section~\ref{sec:applications} we further
motivate our work 
by
presenting
a complete application scenario. Section~\ref{sec:rel} reports 
the
related work, and Section~\ref{sec:conc} concludes the paper.

\section{Formal Notions of Time Granularities}
\label{ch:gran}

Time granularities include very common ones like hours, days, weeks, months and years,
as well as the evolution and specialization of these granularities
for specific contexts or applications.
Trading days, banking days, and academic semesters are just few examples
of specialization of granularities that have become quite common
when describing policies and constraints.

\subsection{Time Granularities}
\label{sec:TimeGranularities}
A comprehensive formal study of time granularities and their relationships can be found in \cite{BJW:book}. In this paper,
we only introduce notions that are essential to show our results.
In particular, we report here the notion of \emph{labeled granularity} which was proposed for the specification of a calendar algebra \cite{BJW:book,NWJ:amai02};
we will show later how any \emph{labeled granularity} can be reduced to a more standard notion of granularity, like the one used by Bettini et al. \citeyear{BWJ:aij02}.

Granularities are defined by grouping sets of instants into \emph{granules}.
For example, each granule of the granularity \texttt{day} specifies the set of instants included in a particular day.
A label is used to refer to a particular granule.
The whole set of time instants is called \emph{time domain}, and for the purpose of this paper the domain can be an arbitrary infinite set with a total order relationship, $\le$.  

\begin{definition}
\label{def:labeledGran}
A {\em labeled granularity} $G$ is a pair $(\Lset{G}, M)$, where
$\Lset{G}$ is a subset of the integers, and $M$ is a mapping from
$\Lset{G}$ to the subsets of the time domain such that for each pair
of integers $i$ and $j$ in $\Lset{G}$ with $i<j$, if
$M(i)\ne\emptyset$ and $M(j)\ne\emptyset$, then (1) each element in
$M(i)$ is less than every element of $M(j)$, and (2) for each integer
$k$ in $\Lset{G}$ with $i<k<j$, $M(k)\ne \emptyset$.
\end{definition}

The former condition guarantees the ``monotonicity'' of the
granularity; the latter is used to introduce the bounds (see
Section~\ref{sec:GranularitiesRelationships}).

We call $\Lset{G}$ the \emph{label set} and for each $i \in \Lset{G}$
we call $G(i)$ a \emph{granule}; if $G(i)\neq \emptyset$ we call it a
\emph{non-empty granule}.
When $\Lset{G}$ is exactly the integers, the granularity is called ``full-integer labeled''. When $\Lset{G}=\mathbb{Z}^+$ we have the same notion of granularity as used in several applications, e.g., \cite{BWJ:aij02}.
For example, following this labeling schema, if we assume to map \texttt{day}(1) to the
subset of the time domain corresponding to January 1, 2001, \texttt{day}(32) would be mapped to February 1, 2001, \texttt{b-day}(6) to January 8, 2001 (the sixth business day), and \texttt{month}(15) to March 2002.
The generalization to arbitrary label sets has been introduced mainly to facilitate conversion operations in the algebra, however our final goal is the conversion of a labeled granularity denoted by a calendar expression into a ``positive-integer labeled'' one denoted by a periodic formula.

\subsection{Granularity Relationships}
\label{sec:GranularitiesRelationships}
Some interesting relationships between granularities follows. The definitions are extended from the ones presented by Bettini et al. \citeyear{BJW:book} to cover the notion of labeled granularity.

\begin{definition}
\label{def:goups-into}
If $G$ and $H$ are labeled granularities, then $G$ is said to \emph{group into} $H$, denoted $G\groups H$, if for each non-empty granule $H(j)$, there exists a (possibly infinite) set $S$ of labels of $G$ such that $H(j)=\bigcup_{i\in S} G(i)$.
\end{definition}
Intuitively, $G \groups H$ means that each granule of $H$ is a union of
some granules of $G$.
For example, \texttt{day}$\groups$\texttt{week} since a week is composed of 7
days and \texttt{day}$\groups$\texttt{b-day} since each business day is a day.

\begin{definition}
\label{def:finer-than}
If $G$ and $H$ are labeled granularities, then $G$ is said to be
\emph{finer than} $H$, denoted $G\finer H$, if for each
granule $G(i)$, there exists a granule $H(j)$ such that $G(i)\subseteq H(j)$.
\end{definition}
For example \texttt{business-day} is finer than \texttt{day}, and also
finer than \texttt{week}.

We also say that $G$ \emph{partitions} $H$ if $G\groups H$ and
$G\finer H$. Intuitively $G$ partitions $H$ if $G\groups H$ and there
are no granules of $G$ other than those included in granules of $H$.
For example, both \texttt{day} and \texttt{b-day} group into
\texttt{b-week} (business week, i.e., the business day in a week), but
\texttt{day} does not partition \texttt{b-week}, while \texttt{b-day}
does.

\begin{definition}
  A labeled granularity $G_1$ is a \emph{label-aligned subgranularity}
  of a labeled granularity $G_2$ if the label set $\Lset{G_1}$ of
  $G_1$ is a subset of the label set $\Lset{G_2}$ of $G_2$ and for
  each $i$ in $\Lset{G_1}$ such that $G_1(i)\ne \emptyset$, we have
  $G_1(i)=G_2(i)$.
\end{definition}
Intuitively, $G_1$ has a subset of the granules of $G_2$ and those granules have 
the same label in the two granularities.

Granularities are said to be \textit{bounded} when $\Lset{G}$ has a
first or last element or when $G(i)=\emptyset$ for some $i\in
\Lset{G}$.
We assume the existence of an unbounded bottom granularity, denoted by
$\bot$ which is full-integer labeled and groups into every other
granularity in the system.

There are time domains such that, given any set of granularities,
it is always possible to find a bottom one; for example, it can be
easily proved that this property holds for each time domain that has
the same cardinality as the integers. On the other hand, the same
property does not hold for other time domains (e.g. the reals).
However, the assumption about the existence of the bottom granularity
is still reasonable since we address problems in which granularities
are defined starting from a bottom one. The definition of a
\emph{calendar} as a set of granularities that have the same bottom
granularity \cite{BJW:book} captures this idea.

\subsection{Granularity Conversions}
When dealing with granularities, we often need to determine the granule (if any) of a granularity $H$ that covers a given granule $z$ of another granularity $G$.
For example, we may wish to find the month (an interval of the absolute time) that includes a given week (another interval of the absolute time). 

This transformation is obtained with the \emph{up} operation. Formally, for each label $z \in \Lset{G}$, $\up{z}_G^H$ is undefined if
$\nexists z' \in \Lset{H}$ s.t. $G(z) \subseteq H(z')$
; otherwise, $\up{z}_G^H=z'$, where $z'$ is the unique index value such that $G(z) \subseteq H(z')$.
The uniqueness of $z'$ is guaranteed by the monotonicity
\footnote{Condition (1) of Definition~\ref{def:labeledGran}.}
of granularities.
As an example, $\up{z}_{\texttt{second}}^{\texttt{month}}$ gives the month that includes the second $z$.
Note that while $\up{z}_{\texttt{second}}^{\texttt{month}}$ is always defined, $\up{z}_{\texttt{week}}^{\texttt{month}}$ is undefined if week $z$ falls between two months. 
Note that if $G \finer H$, then the function $\up{z}_G^H$ is defined for each index value $z$.
For example, since $\texttt{day} \finer \texttt{week}$, $\up{z}_{\texttt{day}}^{\texttt{week}}$ is always defined, i.e.,
for each day we can find the week that contains it.
The notation $\up{z}^{H}$ is used when the source granularity can be left implicit (e.g., when we are dealing with a fixed set of granularities having a distinguished bottom granularity).
  
Another direction of the above transformation is the \emph{down}
operation: Let $G$ and $H$ be granularities such that $G \groups H$,
and $z$ an integer.  Define $\down{z}_G^H$ as the set $S$ of labels of
granules of $G$ such that $\bigcup_{j \in S} G(j) =
H(z)$.\footnote{This definition is different from the one given by Bettini et al
  \citeyear{BJW:book} since it also considers non contiguous granules of $G$.}
This function is useful for finding, e.g., all the days in a month.

\subsection{The Periodical Granules Representation}
\label{sec:ThePeriodicalGranulesRepresentation}
A central issue in temporal reasoning is the possibility of finitely representing infinite granularities. The definition of granularity provided above is general and expressive but it may be impossible to provide a finite representation of some of the granularities. Even labels (i.e., a subset of the integers) do not necessarily have a finite representation.

A solution has been first proposed by Bettini et al. \citeyear{BJW:book}. The idea is that most of the commonly used granularities present a periodical behavior; it means that there is a certain pattern that repeats periodically. This feature has been exploited to provide a method for finitely describing granularities. The formal definition is based on the \emph{periodically groups into} relationship.

\begin{definition}\label{def:pgroup}
A labeled granularity $G$ \emph{groups periodically into} a labeled granularity $H$ ($G \gpinto H$) if $G \groups H$ and there exist positive integers $N$ and $P$ such that 

(1) for each label $i$ of $H$, $i+N$ is a label of $H$ 
unless $i+N$ is greater than the greatest label of $H$,
and

(2) for each label $i$ of $H$, if $H(i)=\bigcup_{r=0}^{k}G(j_r)$ and
$H(i+N)$ is a non-empty granule of $H$ then $H(i+N)=\bigcup_{r=0}^k
G(j_r + P)$, and

(3) if $H(s)$ is the first non-empty granule in $H$ (if exists), then $H(s+N)$
is non-empty.
\end{definition}

The \textit{groups periodically into} relationship is a special case
of the group into characterized by a periodic repetition of the
``grouping pattern'' of granules of $G$ into granules of $H$. Its
definition may appear complicated but it is actually quite simple.
Since $G$ groups into $H$, any granule $H(i)$ is the union of some
granules of $G$; for instance assume it is the union of the granules
$G(a_1), G(a_2), \ldots ,G(a_k)$. Condition (1) ensures that the label
$i+N$ exists (if it not greater than the greatest label of $H$) while
condition (2) ensures that, if $H(i+N)$ is not empty, then it is the
union of $G(a_1 + P), G(a_2 + P), \ldots ,G(a_k + P)$.
We assume that $\forall{r}=0\ldots k$, $(j_r + P) \in \Lset{G}$;
if not, the conditions are considered not satisfied.
Condition (3) simply says that there is at least one of these repetitions.

We call each pair $P$ and $N$ in Definition~\ref{def:pgroup}, a
\emph{period length} and its associated \emph{period label distance}.
We also indicate with $R$ the number of granules of $H$ corresponding
to each groups of $P$ consecutive granules of $\bot$. More formally
$R$ is equal to the number of labels of $H$ greater or equal than $i$
and smaller than $i + N$ where $i$ is an arbitrary label of $H$.  Note
that $R$ is not affected by the value of $i$.

The period length and the period label distance are not unique; more
precisely, we indicate with $P_H^G$ the period length of $H$ in terms
of $G$ and with $N_H^G$ the period label distance of $H$ in terms of
$G$; the form $P_H$ and $N_H$ is used when $G=\bot$.
Note that the period length is an integer value. For simplicity we
also indicate with one period of a granularity $H$ a set of $R$
consecutive granules of $H$.

In general, the \emph{periodically groups into} relationship guarantees that granularity $H$ 
can be finitely described (in terms of granules of $G$).

\begin{definition}
\label{def:finite-representation}
If $G \gpinto H$, then $H$ can be finitely described by providing:
(i) a value for $P$ and $N$;
(ii) the set $\Lset{}^P$ of labels of $H$ in one period of $H$;
(iii) for each $a\in \Lset{}^P$, the finite set $S_a$ of labels of
$G$, such that $H(a) = \bigcup_{i \in S_a}G(i)$;
(iv) the labels of first and last non-empty granules in $H$,
     if their values are not infinite. 
\end{definition}

In this representation, the granules that have labels in $\Lset{}^P$
are the only ones that need to be explicitly represented; we call
these granules the \textit{explicit granules}.
     
     If a granularity $H$ can be represented as a periodic set of
     granules of a granularity $G$, then there exists an infinite
     number of pairs $(P_H^G, N_H^G)$ for which the
     periodically groups into relation is satisfied. If the relation
     is satisfied for a pair $(P, N)$, then it can be proved that it
     can also be satisfied for each pair $(\alpha P, \alpha N)$ with
     $\alpha \in \mathbb{N}^+$.

\begin{definition}\label{def:minimal}
  A periodic representation of a granularity $H$ in terms of $G$ 
  is called \emph{minimal}
  if the period length $P$ used in the representation has
  the smallest value among the period lengths appearing in all the pairs $(P_H^G, N_H^G)$ 
  for which $H$ periodically groups into $G$.
\end{definition}


If $H$ is fully characterized in terms of $G$, it is possible to
derive the composition, in terms of $G$, of any granule of $H$.
Indeed, if $\Lset{}^P$ is the set of labels of $H$ with values in
$\{b, \ldots, b+N_H^G-1\}$, and we assume $H$ to be unbounded, the
description of an arbitrary granule $H(j)$ can be obtained by the
following formula.
Given $j' = [(j-1)\bmod N_{H}^{G}] +1$ and \\[2mm]
\( k=\left \{ \begin{array}{ll} 
    \left(\left\lfloor \frac{b-1}{N_{H}^{G}}\right\rfloor\right) \cdot N_{H}^{G} + j' & \mbox{if } \left(\left\lfloor \frac{b-1}{N_{H}^{G}}\right\rfloor\right) \cdot N_{H}^{G} + j' \geq b\\
    \\
    \left(\left\lfloor \frac{b-1}{N_{H}^{G}}\right\rfloor +1 \right) \cdot N_{H}^{G} +
    j' &\mbox{otherwise}
      \end{array}
      \right.
\)
\\[2mm]
we have
\[
   H(j) = \bigcup_{i \in S_{k}} G\left(P_H^G \cdot \left\lfloor  \frac{j-1}{N_{H}^{G}} \right\rfloor +i - P_H^G \cdot \left\lfloor  \frac{k-1}{N_{H}^{G}} \right\rfloor \right).
\]

\begin{example}
Figure~\ref{fig:pginto} shows granularities $\texttt{day}$ and $\texttt{week\_parts}$ i.e., the granularity that, for each week, contains a granule for the working days and a granule for the weekend.
For the sake of simplicity, we denote $\texttt{day}$ and $\texttt{week\_parts}$ with $D$ and $W$ respectively.
Since $D \gpinto W$, $W$ is fully characterized in terms of $D$.
Among different possible representations, in this example we decide to represent $W$ in terms of $D$ by $P_W^D = 7$, $N_W^D=2$, $\Lset{W}^P=\{3, 4\}$, $S_3 = \{8, 9, 10, 11, 12\}$ and $S_4=\{13, 14\}$.
The composition of each granule of $W$ can then be easily computed; For example the composition of $W(6)$ is given by the formula presented above with $j' = 2$ and $k = 4$. Hence $W(6) = D(7 \cdot 2 + 13 - 7 \cdot 1) \cup D(7 \cdot 2 + 14 - 7 \cdot 1) = D(20) \cup D(21)$.
\end{example}

\begin{figure}[ht!]
        \centering
                \includegraphics[width=.8\columnwidth]{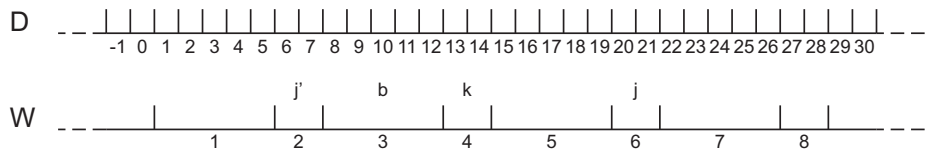}
        \caption{{\em Periodically groups into \/} example}
        \label{fig:pginto}
\end{figure}

\section{Calendar Algebra}
\label{ch:CalAlg}

Several high-level symbolic formalisms have been proposed to represent
granularities \cite{Leban-et-al:86,Niezette:92}.

%
In this work we consider the formalism proposed by Ning et al.
\citeyear{NWJ:amai02} called \textit{Calendar Algebra}.
In this approach a set of
algebraic operations is defined; each operation generates a new
granularity by manipulating other granularities that have already been
generated. The relationships between the operands and the resulting
granularities are thus encoded in the operations.  All granularities
that are generated directly or indirectly from the bottom granularity
form a \textit{calendar}, and these granularities are related to each
other through the operations that define them.  In practice, the
choices for the bottom granularity include \texttt{day},
\texttt{hour}, {\tt second}, {\tt microsecond} and other
granularities, depending on the accuracy required in each application
context.


In the following we illustrate the calendar algebra operations presented by Ning et al. \citeyear{NWJ:amai02} together with some restrictions introduced by Bettini et al. \citeyear{BMW:time04}.

\subsection{The Grouping-Oriented Operations}\label{sec:TheGroupingOrientedOperations}
The calendar algebra consists of the following two kinds of
operations: the {\em grouping-oriented operations} and the {\em
  granule-oriented operations}.
  The grouping-oriented operations
group certain granules of a granularity together to form new granules
in a new granularity.

\subsubsection{The Grouping Operation}
Let $G$ be a full-integer labeled granularity, and $m$ a positive
integer.  The grouping operation $\textsl{Group}_{m}(G)$ generates a
new granularity $G'$ by partitioning the granules of $G$ into
$m$-granule groups and making each group a granule of the resulting
granularity. More precisely, $G' = \textsl{Group}_{m}(G)$ is the
granularity such that for each integer $i$,
\[G'(i) = \bigcup_{j=(i-1)\cdot m+1}^{i\cdot m}G(j).\]
For example, given granularity \texttt{day}, granularity \texttt{week}
can be generated by the calendar algebra expression
$\texttt{week} = \textsl{Group}_{7}(\texttt{day})$
if we assume that $\texttt{day}(1)$ corresponds to Monday, i.e., the
first day of a week.

\subsubsection{The Altering-tick Operation}
\label{sub:altering-tick}
Let $G_{1}$, $G_{2}$ be full-integer labeled granularities, and $l$,
$k$, $m$ integers, where $G_{2}$ partitions $G_{1}$, and $1 \leq l
\leq m$.  The altering-tick operation $Alter_{l,k}^{m}(G_{2}, G_{1})$
generates a new granularity by periodically expanding or shrinking
granules of $G_{1}$ in terms of granules of $G_{2}$.  Since $G_{2}$
partitions $G_{1}$, each granule of $G_{1}$ consists of some
contiguous granules of $G_{2}$.  The granules of $G_{1}$ can be
partitioned into $m$-granule groups such that $G_{1}(1)$ to $G_{1}(m)$
are in one group, $G_{1}(m+1)$ to $G_{1}(2m)$ are in the following
group, and so on.  The goal of the altering-tick operation is to
modify the granules of $G_{1}$ so that the $l$-th granule of every
$m$-granule group will have $\vert k \vert$ additional (or fewer
when $k < 0$) granules of $G_{2}$.  For example, if $G_1$ represents
30-day groups (i.e., $G_1=\textsl{Group}_{30}(\texttt{day})$) and we
want to add a day to every $3$-rd month (i.e., to make March to have
31 days), we may perform $\textsl{Alter}_{3,1}^{12}(\texttt{day},
G_{1})$.


The altering-tick operation can be formally described as follows. For each integer $i$ such that $G_1(i)\ne\emptyset$, let $b_{i}$ and $t_{i}$ be the integers such that $G_{1}(i) = \cup_{j=b_{i}}^{t_{i}}G_{2}(j)$ (the integers $b_{i}$ and $t_{i}$ exist because $G_{2}$ partitions $G_{1}$).
Then $G'= \textsl{Alter}_{l,k}^{m}(G_{2}, G_{1})$ is the granularity such that for each integer $i$, let $G'(i)=\emptyset$ if $G_1(i)=\emptyset$, and otherwise let
\[G'(i) = \bigcup_{j=b_{i}'}^{t_{i}'} G_{2}(j),\] where 
\[b_{i}'=\left \{ \begin{array}{ll}
      b_{i} + (h-1) \cdot k, & \mbox{if } i = (h-1)\cdot m + l,\\
      b_{i} + h \cdot k,     &\mbox{otherwise},
      \end{array}
      \right.
\]
\[t_{i}' = t_{i} + h\cdot k,\]
and 
\[h = \left\lfloor \frac{i-l}{m} \right\rfloor + 1.\]

\begin{example}
Figure~\ref{fig:AlterExample} shows an example of the \texttt{Alter} operation.
Granularity $G_1$ is defined by $G_1=\textsl{Group}_{5}(G_2)$
and granularity $G'$ is defined by $G' = \textsl{Alter}_{2, -1}^{2} (G_2, G_1)$,
which means shrinking the second one of every two granules of $G_1$ by one granule of $G_2$.
\end{example}

\begin{figure}[ht]
        \centering
                \includegraphics[width=.8\columnwidth]{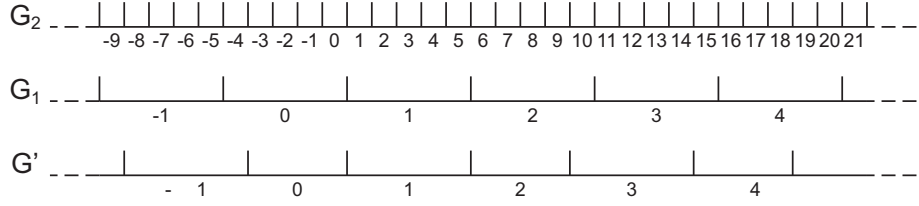}
        \caption{\textsl{Altering-tick} operation example}
        \label{fig:AlterExample}
\end{figure}


The original definition of altering-tick given by Ning et al. \citeyear{NWJ:amai02} as reported above, has the following problems when an arbitrary negative value for $k$ is used: (1) It allows the definition of a $G'$ that is not a full-integer labeled granularity and (2) It allows the definition of a $G'$ that does not even satisfy the definition of granularity.
In order to avoid this undesired behavior, we impose the following restriction:
\[
k > -(mindist(G1,2,G2)-1)
\]
where $mindist()$ is formally defined by Bettini et al. \citeyear{BJW:book}.

Intuitively, $mindist(G1,2,G2)$ represents the minimum distance (in terms of granules of $G2$) between two consecutive granules of $G1$.

\subsubsection{The Shift Operation}
Let $G$ be a full-integer labeled granularity, and $m$ an integer.
The shifting operation $\textsl{Shift}_{m}(G)$ generates a new granularity $G'$ by shifting the labels of $G$ by $m$ positions.
More formally, $G' = \textsl{Shift}_{m}(G)$ is the granularity such that for each integer
$i$, $G'(i) = G(i-m)$.
Note that $G'$ is also full-integer labeled.

\subsubsection{The Combining Operation}
Let $G_{1}$ and $G_{2}$ be granularities with label sets $\Lset{G_1}$ and $\Lset{G_2}$ respectively.
The combining operation $\textsl{Combine} (G_{1}, G_{2})$ generates a new granularity $G'$ by combining all the granules of $G_{2}$ that are included in one granule of $G_{1}$ into one granule of $G'$.
More formally, for each $i \in {\cal L}_1$, let $s(i)=\emptyset$ if $G_{1}(i)=\emptyset$,
and otherwise let
$s(i)=\{j\in \Lset{G_2} | \emptyset\ne G_2(j)\subseteq G_1(i)\}$.
Then $G' = \textsl{Combine} (G_{1}, G_{2})$ is the granularity with the label set 
$\Lset{G'}=\{i\in \Lset{G_1}|s(i)\ne\emptyset\}$
such that for each $i$ in $\Lset{G'}$,
$G'(i) = \bigcup_{j\in s(i)} G_{2}(j)$.

As an example, given granularities \texttt{b-day} and \texttt{month},
the granularity for business months can be generated by
$\texttt{b-month} = \textsl{Combine} (\texttt{month}, \texttt{b-day}).$

\subsubsection{The Anchored Grouping Operation}
Let $G_{1}$ and $G_{2}$ be granularities with label sets $\Lset{G_1}$ and $\Lset{G_2}$ respectively, where $G_{2}$ is a label-aligned subgranularity of $G_{1}$, and $G_{1}$
is a full-integer labeled granularity. 
The anchored grouping operation $\textsl{Anchored-group}(G_{1}, G_{2})$ generates a new granularity $G'$ by combining all the granules of $G_{1}$ that are between two granules of
$G_{2}$ into one granule of $G'$.
More formally, $G' = \textsl{Anchored-group}(G_{1}, G_{2})$ is the granularity with the label set $\Lset{G'} = \Lset{G_2}$ such that for each $i \in \Lset{G'}$, 
$G'(i) = \bigcup_{j=i}^{i'-1}G_{1}(j)$
where $i'$ is the next label of $G_{2}$ after $i$.


For example, each academic year at a certain university begins on the last Monday in August, and ends on the day before the beginning of the next academic year. Then, the granularity corresponding to the academic years can be generated by
$AcademicYear = \textsl{Anchored-group}(\texttt{day}, \texttt{lastMondayOfAugust})$.

\subsection{The Granule-Oriented Operations}
Differently from the grouping-oriented operations,
the granule-oriented operations do not
modify the granules of a granularity,
but rather enable the selection of the
granules that should remain in the new granularity.

\label{sec:TheGranuleOrientedOperations}

\subsubsection{The Subset Operation}

Let $G$ be a granularity with label set $\Lset{G}$, and $m, n$ integers such that $m \leq n$.
The subset operation $G' = \textsl{Subset}^{n}_{m}(G)$ generates a new granularity $G'$ by taking all the granules of $G$ whose labels are between $m$ and $n$.
More formally, $G' = \textsl{Subset}^{n}_{m}(G)$ is the granularity with the label set $\Lset{G'} = \{ i\in \Lset{G} \mid m \leq i \leq n\}$, and for each $i \in \Lset{G'}$, 
$G'(i) = G(i)$.
For example, given granularity \texttt{year}, all the years in the 20th century 
can be generated by
$\texttt{20CenturyYear} = \textsl{Subset}_{1900}^{1999}(\texttt{year})$.
Note that $G'$ is a label-aligned subgranularity of $G$,
and $G'$ is not a full-integer labeled granularity even if $G$ is. We also allow the extensions of setting $m=-\infty$ or $n=\infty$ with semantics properly extended.

\subsubsection{The Selecting Operations}
The selecting operations are all binary operations. They generate new granularities by selecting granules from the first operand in terms of their relationship with the granules of the second operand.
The result is always a label-aligned subgranularity of the first operand granularity. 

There are three selecting operations: \textsl{select-down, select-up} and \textsl{select-by-intersect}. To facilitate the description of these operations, the $\Delta_{k}^{l}(S)$ notation is used. Intuitively, if $S$ is a set of integers, $\Delta_{k}^{l}(S)$ selects $l$ elements starting from the $k$-th one (for a formal description of the $\Delta$ operator see \cite{NWJ:amai02}).

%

\vskip 6pt
\noindent {\em Select-down operation.}
\index{calendar operation!selecting!select-down}
For each granule $G_{2}(i)$, there exits a set of granules of $G_{1}$ that is contained in $G_{2}(i)$.
The operation $\textsl{Select-down}_{k}^{l} (G_{1}, G_{2})$, where $k\neq 0$ and $l > 0$ are 
integers, selects granules of $G_{1}$ by using $\Delta_{k}^{l}(\cdot)$ on each set of granules (actually their labels) of $G_1$ that are contained in one granule of $G_2$.
More formally, $G' = \textsl{Select-down}_{k}^{l}(G_{1}, G_{2})$ is the granularity with the
label set
\[\Lset{G'} = \cup_{i \in \Lset{G_2}} \Delta_{k}^{l}( \{j \in \Lset{G_1} \mid \emptyset\ne G_{1}(j) \subseteq G_{2}(i) \}),\]
and for each $i \in \Lset{G'}$,
$G'(i) = G_{1}(i)$.
For example, Thanksgiving days are the fourth Thursdays of all Novembers; if \texttt{Thursday} and \texttt{November} are given, it can be generated by $\texttt{Thanksgiving} = \textsl{Select-down}_{4}^{1} (\texttt{Thursday}, \texttt{November})$.

\vskip 6pt
\noindent{\em Select-up operation.}
The select-up operation $\textsl{Select-up}(G_{1}, G_{2})$ generates a new granularity $G'$ by selecting the granules of $G_{1}$ that contain one or more granules of $G_{2}$.
More formally, $G' = \textsl{Select-up} (G_{1}, G_{2})$ is the granularity with the label set
\[\Lset{G'}= \{ i \in \Lset{G_1} \vert \exists j \in \Lset{G_2} (\emptyset\ne G_{2}(j) \subseteq G_{1}(i)),\}\] 
and for each $i \in \Lset{G'}$, 
$G'(i) = G_{1}(i)$.
For example, given granularities \texttt{Thanksgiving} and \texttt{week}, 
the weeks that contain Thanksgiving days can be defined by
$\texttt{ThanxWeek} = \textsl{Select-up}(\texttt{week}, \texttt{Thanksgiving})$.

\vskip 6pt
\noindent {\em Select-by-intersect operation.}
For each granule $G_{2}(i)$, there may exist a set of granules of $G_{1}$, each intersecting $G_{2}(i)$.
The $\textsl{Select-by-intersect}_{k}^{l} (G_{1}, G_{2})$ operation, where $k\neq 0$
and $l > 0$ are integers, selects granules of $G_{1}$ by applying $\Delta_{k}^{l}(\cdot)$
operator to all such sets, generating a new granularity $G'$.
More formally, $G' = \textsl{Select-by-intersect}_{k}^{l}(G_{1}, G_{2})$ is the granularity
with the label set 
\[\Lset{G'} = \cup_{i \in \Lset{G_2}} \Delta_{k}^{l}(\{j \in \Lset{G_1} \mid G_{1}(j) \cap 
G_{2}(i) \neq \emptyset\}),\] 
and for each $i \in \Lset{G'}$,
$G'(i) = G_{1}(i)$.
For example, given granularities \texttt{week} and \texttt{month}, the granularity consisting of the first week of each month (among all the weeks intersecting the month) can be generated by 
$\texttt{FirstWeekOfMonth} = \textsl{Select-by-intersect}_{1}^{1} (\texttt{week}, \texttt{month})$.

\subsubsection{The Set Operations}

In order to have the set operations as a part of the calendar algebra
and to make certain computations easier, we restrict the operand
granularities participating in the set operations so that the result
of the operation is always a valid granularity: the set operations can
be defined on $G_{1}$ and $G_{2}$ only if there exists a granularity
$H$ such that $G_{1}$ and $G_{2}$ are both label-aligned
subgranularities of $H$.  In the following, we describe the union,
intersection, and difference operations of $G_{1}$ and $G_{2}$,
assuming that they satisfy the requirement.

\vskip 6pt
\noindent{\em Union.}
\index{calendar operation!set operation!union} The union operation
$G_{1}\cup G_{2}$ generates a new granularity $G'$ by collecting all
the granules from both $G_{1}$ and $G_{2}$. More formally, $G'=
G_{1}\cup G_{2}$ is the granularity with the label set
$\Lset{G'}=\Lset{G_1} \cup \Lset{G_2}$, and for each $i \in
\Lset{G'}$,
\[G'(i) = \left \{\begin{array}{ll}
           G_{1}(i),    &i \in {\cal L}_{1},\\
           G_{2}(i),    &i \in {\cal L}_{2} - {\cal L}_{1}.\\
           \end{array}
           \right.
\]
For example, given granularities \texttt{Sunday} and
\texttt{Saturday}, the granularity of the weekend days can be
generated by $\texttt{WeekendDay} = \texttt{Sunday} \cup
\texttt{Saturday}$.

\vskip 6pt\noindent{\em Intersection.}  \index{calendar operation!set
  operation!intersection} The intersection operation $G_{1}\cap G_{2}$
generates a new granularity $G'$ by taking the common granules from
both $G_{1}$ and $G_{2}$.  More formally, $G'= G_{1}\cap G_{2}$ is the
granularity with the label set $\Lset{G'}=\Lset{G_1} \cap \Lset{G_2}$,
and for each $i \in \Lset{G'}$, $G'(i) = G_{1}(i)$ (or equivalently
$G_{2}(i)$).

\vskip 6pt\noindent{\em Difference.}
\index{calendar operation!set operation!difference}
The difference operation $G_{1} \setminus G_{2}$ generates a new granularity $G'$ by excluding the granules of $G_{2}$ from those of $G_{1}$.
More formally, $G'= G_{1} \setminus G_{2}$ is the granularity with the label set $\Lset{G'}=\Lset{G_1} \setminus \Lset{G_2}$, and for each $i \in \Lset{G'}$,
$G'(i) = G_{1}(i)$. 

\section{From Calendar Algebra to Periodical Set}
\label{ch:CalAlg2PSet}


In this section we first describe the overall conversion process and then we report the formulas specific for the conversion of each calendar algebra operation. Finally, we present a procedure for relabeling the resulting granularity, a sketch complexity analysis and some considerations about the period length minimality.

\subsection{The Conversion Process}
\label{sec:TheConversionProcess}
Our final goal is to provide a correct and effective way to convert calendar expressions into periodical representations. Under appropriate limitations, for each calendar algebra operation, if the periodical descriptions of the operand granularities are known, it is possible to compute the periodical characterization of the resulting granularity.

This result allows us to calculate, for any calendar, the periodical description of each granularity in terms of the bottom granularity. In fact, by definition, the bottom granularity is fully characterized; hence it is possible to compute the periodical representation of all the granularities that are obtained from operations applied to the bottom granularity. Recursively, the periodical description of all the granularities can be obtained.

The calendar algebra presented in the previous section can represent all the granularities that are periodical with finite exceptions (i.e., any granularity $G$ such that bottom groups periodically with finite exceptions into $G$). Since with the periodical representations defined in Section~\ref{ch:gran} it is not possible to express the finite exceptions, we need to restrict the calendar algebra so that it cannot represent them. This implies allowing the \textsl{Subset} operation to be only used as the last step of deriving a granularity.
Note that in the calendar algebra presented by Ning et al. \citeyear{NWJ:amai02} there was an extension to the altering-tick operation to allow the usage of $\infty$ as the $m$ parameter (i.e., $G'=\textsl{Alter}_{l,k}^{\infty}(G_2, G_1)$); the resulting granularity has a single exception hence is not periodic. This extension is disallowed here in order to generate periodical granularities only (without finite exceptions).

The conversion process can be divided into three steps: in the first one the period length and period label distance are computed; in the second we derive the set $\Lset{}^P$ of labels in one period, and in the last one the composition of the explicit granules is computed. For each operation we identify the correct formulas and algorithms for the three steps.

The \textbf{first step} consists in computing the period length and the period label distance of the resulting granularity. Those values are calculated as a function of the parameters (e.g. the ``grouping factor'' $m$,  in the \textsl{Group} operation) and the operand granularities (actually their period lengths and period label distances).
%


The \textbf{second step} in the conversion process is the identification of the label set of the resulting granularity. In Section~\ref{sec:ThePeriodicalGranulesRepresentation} we pointed out that in order to fully characterize a granularity it is sufficient to identify the labels in any period of the granularity. In spite of this theoretical result, to perform the computations required by each operation we need the explicit granules of the operand granularities to be ``aligned''. There are two possible approaches: the first one consist in computing the explicit granules in any period and then recalculate the needed granules in the correct position in order to eventually align them. The second one consists in aligning all the periods containing the explicit granules with a fixed granule in the bottom granularity. After considering both possibilities, for performance reasons, we decided to adopt the second approach. We decided to use $\bot(1)$ as the ``alignment point'' for all the granularities. A formal definition of the used formalism follows.

Let $G$ be a granularity and $i$ be the smallest positive integer such that $\up{i}^G$ is defined.
We call $l_G = \up{i}^G$ and $\Lbar{G}$ the set of labels of $G$ contained in $l_G \ldots l_G+N_G -1$.
Note that this definition of $\Lbar{G}$ is an instance of the definition of $\Lset{}^P$ given in Section~\ref{sec:ThePeriodicalGranulesRepresentation}.
The definition of $\Lbar{G}$ provided here is useful for representing $G$ and actually the final goal of this step is to compute $\Lbar{G}$; however $\Lbar{G}$ is not suitable for performing the computations. The problem is that if $G(l_G)$ starts before $\bot(1)$ (i.e., $min(\down{l_G}^G)<1$) then the granule $G(l_G+N_G)$ begins at $P_{G}$ or before $P_G$, and hence $G(l_G + N_{G})$ is necessary for the computations; however $l_G+N_G \notin \Lbar{G}$.

To solve the problem we introduce the symbol $\Lcap{G}$ to represent the set of all labels of granules of $G$ that cover one in $\bot(1) \ldots \bot(P_G)$. It is easily seen that if $G(l_G)$ does not cover $\bot(0)$, then $\Lcap{G} = \Lbar{G}$, otherwise $\Lcap{G} = \Lbar{G} \cup \{l_G + N_G\}$. Therefore the conversion between $\Lbar{}$ and $\Lcap{}$ and vice versa is immediate.

The notion of $\Lcap{}$ is still not enough to perform the computations. The problem is that when a granularity $G$ is used as an operand in an operation, the period length of the resulting granularity $G'$ is generally bigger than the period length of $G$. Therefore it is necessary to extend the notion of $\Lcap{G}$ to the period length $P_{G'}$ of $G'$ using $P_{G'}$ in spite of $P_G$ in the definition of $\Lcap{}$. The symbol used for this notion is $\Lcap{G}^{P_{G'}}$.

The idea is that when $G$ is used as the operand in an operation that generates $G'$, $\Lcap{G}^{P_{G'}}$ is computed from $\Lbar{G}$. This set is then used by the formula that we provide below to compute $\Lbar{G'}$.

The computation of $\Lbar{G'}$ is performed as follows: if $G'$ is defined by an operation that returns a full-integer labeled granularity, then it is sufficient to compute the value of $l_G'$. Indeed it is easily seen that $\Lbar{G'}=\{i \in \mathbb{Z} | l_G' \leq i \leq l_G' + N_{G'} -1 \}$. If $G'$ is defined by any other algebraic operation, we provide the formulas to compute $\Lcap{G'}$; from $\Lcap{G'}$ we easily derive $\Lbar{G'}$. 

\begin{example}
 
Figure~\ref{fig:LbarLset} shows granularities $\bot$,
  $G$ and $H$; it is clear that $P_G = P_H = 4$ and $N_G=N_H=3$.
  Moreover, $l_G = l_H = 6$ and therefore $\Lbar{G} = \Lbar{H} =
  \{6,7\}$. Since $0 \notin \down{6}^G$ then $\Lcap{G} = \Lbar{G}$. On
  the other hand, since $0 \in \down{6}^H$, then $\Lcap{H} = \Lbar{H}
  \cup \{6 + 3\}$.
 
 Suppose that a granularity $G'$ has period
  length $P_{G'} = 8$; then $\Lcap{G}^{P_G'} = \{6,7,9,10\}$ and
  $\Lcap{H}^{P_{G'}} = \{6,7,9,10,12\}$.
\end{example}

\begin{figure}[ht]
        \centering
                \includegraphics[width=.5\columnwidth]{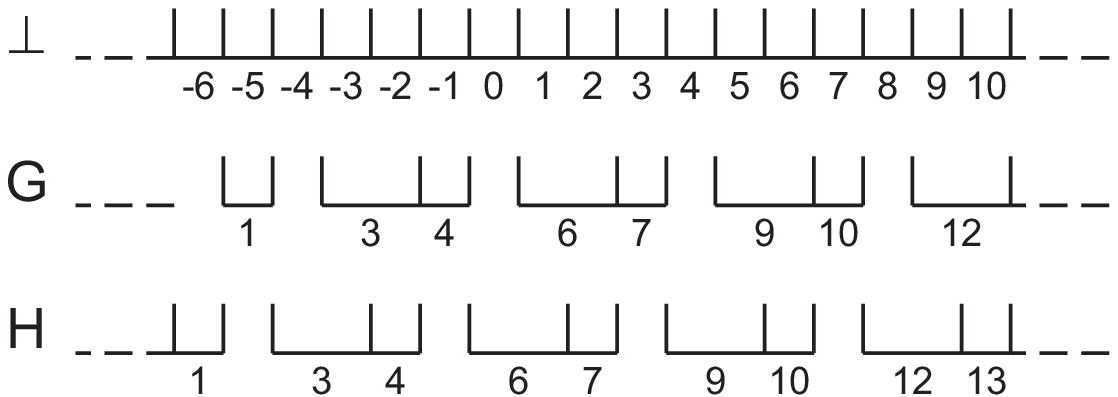}
        \caption{$\Lbar{}$, $l$, $\Lcap{}$ and $\Lcap{}^{P_{G'}}$ examples}
        \label{fig:LbarLset}
\end{figure}

The \textbf{third} (and last) \textbf{step} of the conversion process is the computation of the composition of the explicit granules. Once $\Lbar{G'}$ has been computed, it is sufficient to apply, for each label of $\Lbar{G'}$ the formulas presented in Chapter~\ref{ch:CalAlg}.

In Sections~\ref{sub:group} to \ref{sub:setOp} we show, for each
calendar algebra operation, how to compute the first and second
conversion steps.

\subsection{Computability Issues}
\label{sec:computability}
In some of the formulas presented below it is necessary to compute the set $S$ of labels of a granularity $G$ such that $\forall i \in S \; G(i) \subseteq H(j)$ where $H$ is a granularity and $j$ is a specific label of $H$. Since $\Lset{G}$ contains an infinite number of labels, it is not possible to check, $\forall i \in \Lset{G}$ if $G(i) \subseteq H(j)$.
However it is easily seen that $\forall i \in S \; \exists k$ s.t. $G(\up{k}^G) \subseteq H(j)$.
Therefore $\forall i \in S \; \exists k$ s.t. $G(\up{k}^G)$ is defined and  $k \in \down{j}^H$.

Therefore we compute the set $S$ by considering all the labels $i$ of $\Lset{G}$ s.t. $\exists n \in \down{j}^H$ s.t. $\up{n}^{G} = i$ and $G(i) \subseteq H(j)$. Since the set $\down{j}^H$ is finite\footnote{With the calendar algebra it is not possible to define granularities having granules that maps to an infinite set of time instants.}, the computation can be performed in a finite time.
The consideration is analogous if $S$ is the set such that $\forall i \in S \; G(i) \supseteq H(j)$ or $\forall i \in S \; (G(i) \cap H(j) \neq \emptyset)$.

\subsection{The Group Operation}
\label{sub:group}
\begin{prop}
\label{prop:GroupOperation}
If $G'=Group_{m}(G)$, then:
\begin{enumerate}
        \item $P_{G'}=\frac{P_{G} \cdot m}{GCD(m, N_G)}$ and $N_{G'} = \frac{N_{G}}{GCD(m, N_G)}$;
        \item $l_{G'}=\left( \left\lfloor \frac{l_{G}-1}{m}\right\rfloor + 1 \right)$;
        \item $\forall i \in \Lbar{G'} \; G'(i) = \bigcup_{j=(i-1) \cdot m +1} ^ {i \cdot m} G(j)$.
\end{enumerate}
 
\end{prop}

\begin{example}
  Figure~\ref{fig:GroupOperation} shows an example of the group operation: $G' = \textsl{Group}_3(G)$. Since $P_G=1$ and $N_G=1$, then $P_{G'}=3$ and $N_G=1$. Moreover, since $\Lbar{G}=\{-7\}$, then $l_G=-7$ and therefore $l_{G'}=-2$ and $\Lbar{G'}=\{-2\}$. Finally $G'(-2) = G(-8) \cup G(-7) \cup G(-6)$ i.e., $G'(-2) = \bot(0) \cup \bot(1) \cup \bot(2)$.
\end{example}

\begin{figure}[ht]
        \centering
                \includegraphics[width=.8\columnwidth]{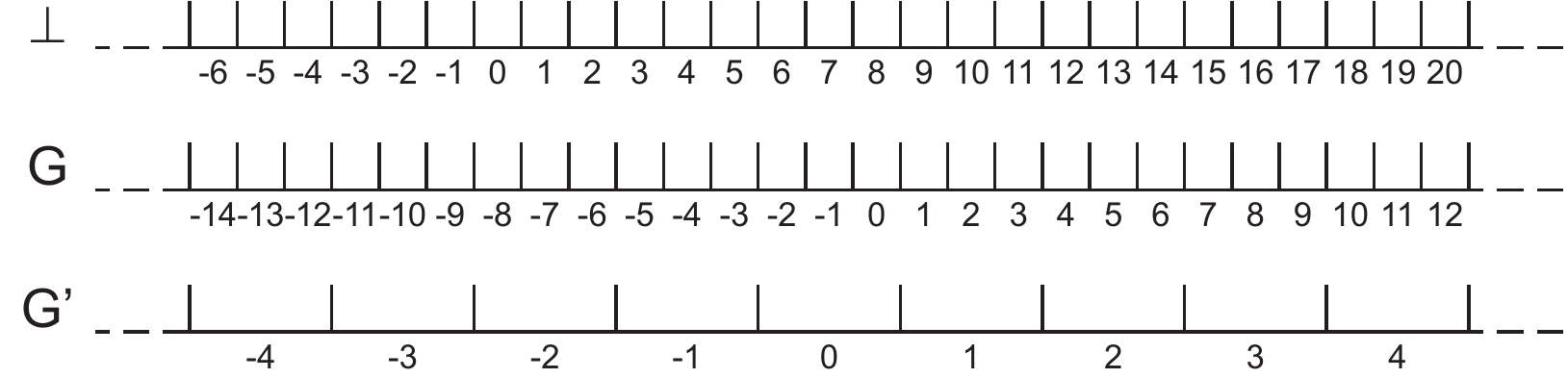}
        \caption{\textsl{Group} operation example}
        \label{fig:GroupOperation}
\end{figure}

\subsection{The Altering-tick Operation}

\begin{prop}
\label{pr:alterOperation}
If $G'=\textsl{Alter}_{l,k}^{m}(G_2, G_1)$ then:
\begin{enumerate}
        \item
\[
N_{G'} = lcm \left( N_{G_1}, m, \frac{P_{G_2} \cdot N_{G_1}}{GCD(P_{G_2} \cdot N_{G_1}, P_{G_1})}, \frac{N_{G_2} \cdot m}{GCD(N_{G_2} \cdot m, |k|)} \right)
\]
\noindent and
\[
P_{G'} = \left( \frac{N_{G'} \cdot P_{G_1} \cdot N_{G_2}}{N_{G_1} \cdot P_{G_2}} + \frac{N_{G'} \cdot k}{m}\right) \cdot \frac{P_{G_2}}{N_{G_2}}
\]
        \item $l_{G'}=\up{l_{G_2}}^{G'}_{G_2}$; 
        \item $\forall i \in \Lbar{G'} \; G'(i) = \bigcup_{j=b_i'} ^ {t_i'} G(j)$ where $b_i'$ and $t_i'$ are defined in Section~\ref{sub:altering-tick}.
\end{enumerate}

\end{prop}
Referring to step 2., note that when computing $l_{G'}$ the explicit characterization of the granules of $G'$ is still unknown. To perform the operation $\up{l_{G_2}}^{G'}_{G_2}$ we need to know at least the explicit granules of one of its periods. We choose to compute the granules labeled by $1 \ldots N_{G'}$. When $l_{G'}$ is derived, the granules labeled by $l_{G'} \ldots l_{G'} + N_{G'} - 1$ will be computed so that the explicit granules are aligned to $\bot(1)$ as required.

\begin{example}
  Figure~\ref{fig:AlterOperation} shows an example of the altering-tick operation: $G' = \textsl{Alter}_{2,1}^{3}(G_2, G_1)$. Since $P_{G_1}=4$, $N_{G_1}=1$, $P_{G_2}=4$ and $N_{G_2}=2$, then $N_{G'}=6$ and $P_{G'}=28$. Moreover, since $\Lbar{G_2}=\{-10, -9\}$, then $l_{G_2}=-10$ and therefore $l_{G'}=\up{-10}^{G'}_{G_2}=-4$ and hence $\Lbar{G_2}=\{-4, -3, \ldots, 0, 1\}$. Finally $G'(-4) = G_1(-11) \cup G_1(-10) \cup G_1(-9) = \bot(-1) \cup \bot(0) \cup \bot(1) \cup \bot(3) \cup \bot(4)$; analogously we derive $G'(-3)$, $G'(-2)$, $G'(-1)$, $G'(0)$ and $G'(1)$.
\end{example}

\begin{figure}[ht]
        \centering
                \includegraphics[width=\columnwidth]{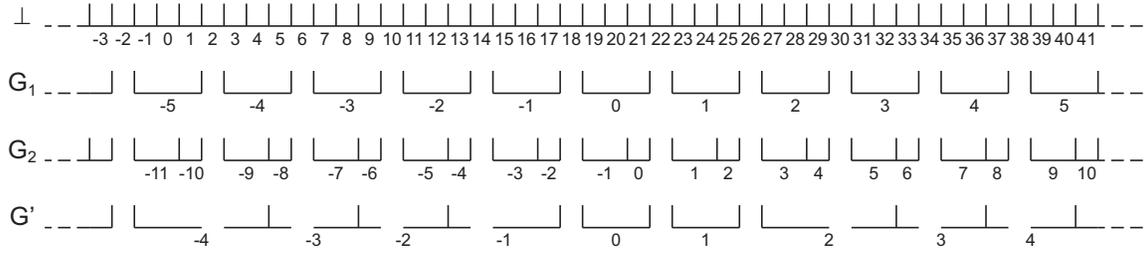}
        \caption{\textsl{Alter} operation example}
        \label{fig:AlterOperation}
\end{figure}

\subsection{The Shift Operation}
\begin{prop}
\label{pr:shiftOperation}
If $G'=\textsl{Shift}_m(G)$, then:
\begin{enumerate}
        \item $P_{G'} = P_{G_1}$ and $N_{G'} = N_{G_1}$;
        \item $l_{G'}=l_G + m$;
        \item $\forall i \in \Lbar{G'} \; G'(i) = G(i - m)$.
\end{enumerate}
\end{prop}

\begin{example}
The shifting operation can easily model time differences. Suppose granularity
\texttt{USEast-Hour} stands for the hours of US Eastern Time. Since the hours of the US Pacific Time are 3 hours later than those of US Eastern Time, the hours of US Pacific Time can be generated by \texttt{USPacific-Hour}$ = \textsl{Shift}_{-3}$(\texttt{USEast-Hour}).
\end{example}

\begin{figure}[ht]
        \centering
                \includegraphics[width=.55\columnwidth]{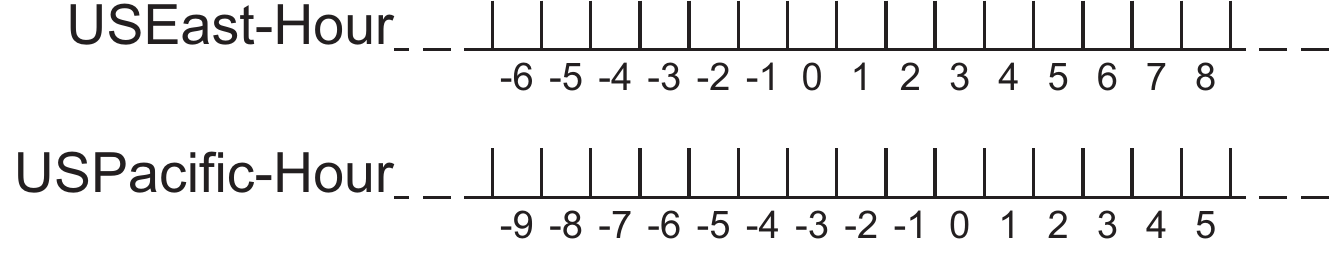}
        \caption{\textsl{Shift} operation example}
        \label{fig:ShiftOperation}
\end{figure}

\subsection{The Combining Operation}
\begin{prop}
\label{pr:combineOperation}
Given $G'=\textsl{Combining}(G_1, G_2)$, then:
\begin{enumerate}
        \item $P_{G'} = lcm (P_{G_1}, P_{G_2})$ and $N_{G'} = \frac{lcm (P_{G_1}, P_{G_2})N_{G_1}}{P_{G_1}}$;
        \item  $\forall i \in \Lcap{G_1}^{P_{G'}}$ let be $\widetilde{s}(i) = \{j \in \Lcap{G_2}^{P_{G'}} | \emptyset \neq G_2(j) \subseteq G_1(i) \}$; then $\Lcap{G'} = \{ i \in \Lcap{G_1}^{P_{G'}} | \widetilde{s}(i) \neq \emptyset\}$;
                \item $\forall i \in \Lbar{G'} \; G'(i) = \bigcup_{j \in s(i)}G_2(j)$.
\end{enumerate}
\end{prop}

\begin{example}
  Figure~\ref{fig:CombineOperation} shows an example of the combining operation: $G' = \textsl{Combine}(G_1, G_2)$. Since $P_{G_1}=6$, $N_{G_1}=2$, $P_{G_2}=4$ and $N_{G_2}=2$, then $P_{G'}=12$ and $N_{G'}=4$. Moreover, since $\Lbar{G_1}=\{1\}$ and $0 \in \down{1}^{G_1}$, then $\Lcap{G_1} = \{1,3\}$ and hence $\Lcap{G_1}^{P_{G'}} = \{1,3,5\}$. Since $\tilde{s}(i) \neq \emptyset$ for $i \in \{1,3,5\}$, then $\Lcap{G'}=\{1,3,5\}$; moreover, since $0 \in \down{1}^{G'}$, then $\Lbar{G'} = \{1,3\}$.
Finally $s(1) = \{-1,0\}$ and $s(3) = \{2, 3\}$; consequently, $G'(1) = G_2(-1) \cup G_2(0)$ i.e., $G'(1) = \bot(-1) \cup \bot(0) \cup \bot(1)$ and $G'(3) = G_2(2) \cup G_2(3)$ i.e., $G'(3) = \bot(4) \cup \bot(5) \cup \bot(7)$.
\end{example}

\begin{figure}[ht]
        \centering
                \includegraphics[width=.7\columnwidth]{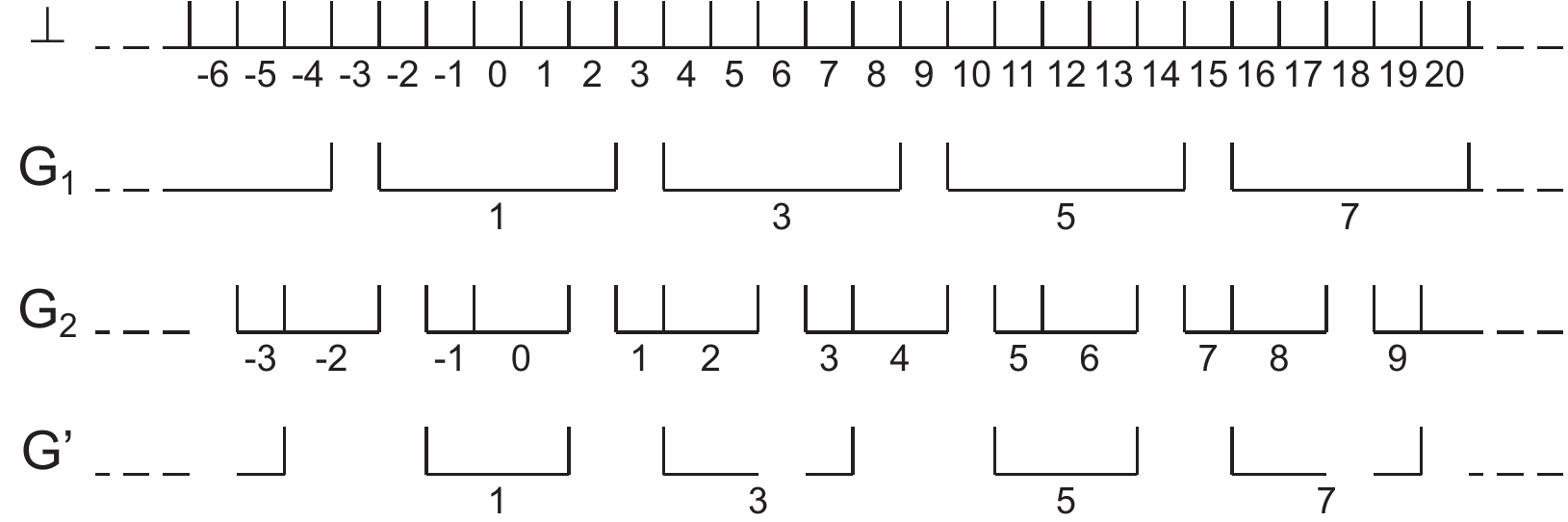}
        \caption{\textsl{Combine} operation example}
        \label{fig:CombineOperation}
\end{figure}

\subsection{The Anchored Grouping Operation}
\begin{prop}
\label{pr:anchorOperation}
Given $G'=\textsl{Anchored-group}(G_1, G_2)$, then:
\begin{enumerate}
        \item $P_{G'}=lcm(P_{G_1}, P_{G_2})$ and $N_{G'}= \frac{lcm(P_{G_1}, P_{G_2})\cdot N_{G_2}}{P_{G_2}}$;
        \item 
        \[\Lcap{G'}=\left \{ \begin{array}{ll}
      \Lcap{G_2}^{P_{G'}}, & \mbox{if } l_{G_2}=l_{G_1},\\
      \ \{l'_{G_2}\} \cup \Lcap{G_2}^{P_{G'}},     &\mbox{otherwise},
      \end{array}
      \right.
\]
        where $l'_{G_2}$ is the greatest among the labels of $\Lset{G_2}$ that are smaller than $l_{G_2}$.
        
        \item $\forall i \in \Lbar{G'} \; G'(i) = \bigcup_{j=i}^{i'-1}G_1(j)$ where $i'$ is the next label of $G_2$ after $i$.
\end{enumerate}
\end{prop}

\begin{example}
  Figure~\ref{fig:AnchorOperation} shows an example of the anchored grouping operation: the \texttt{USweek} (i.e., a week starting with a \texttt{Sunday}) is defined by the operation
  \textsl{Anchored-group}(\texttt{day}, \texttt{Sunday}). Since $P_{\texttt{day}} = 1$ and $P_{\texttt{Sunday}} = 7$, then the period length of $\texttt{USweek}$ is 7.
  Moreover since $l_{\texttt{day}} = 11$, $l_{\texttt{Sunday}}=14$ and $\Lcap{\texttt{Sunday}}^{P_{\texttt{USweek}}}=\{14\}$, then $\Lcap{\texttt{USweek}}=\{7\} \cup \{14\}$.
  Clearly, since $0 \in \down{7}^{\texttt{USweek}}$ then $\Lbar{\texttt{USweek}} = \{7\}$.
Finally, $\texttt{USweek}(7)=\bigcup_{j=7}^{13} \texttt{day}(j)=\bigcup_{k=-3}^{3}\bot(k)$.
\end{example}

\begin{figure}[ht]
        \centering
                \includegraphics[width=.9\columnwidth]{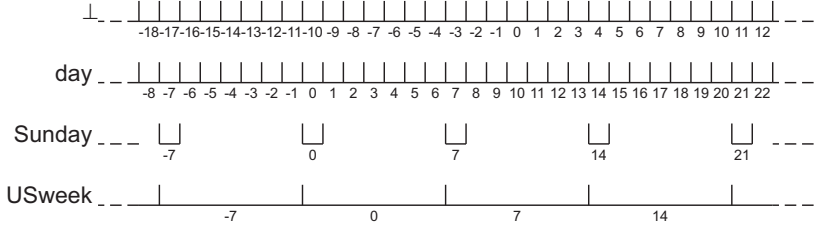}
        \caption{\textsl{Anchored Grouping} operation example}
        \label{fig:AnchorOperation}
\end{figure}

\subsection{The Subset Operation}
The \textsl{Subset} operation only modifies the operand granularity by introducing the bounds. The period length, the period label distance, $\Lbar{}$ and the composition of the explicit granules are not affected.

\subsection{The Selecting Operations}

\subsubsection{The Select-down Operation}
\begin{prop}
\label{pr:selectDownOperation}
Given $G'=\textsl{Select-down}_k^l(G_1, G_2)$, then:
\begin{enumerate}
        \item $P_{G'}=lcm (P_{G_1}, P_{G_2})$ and $N_{G'}=\frac{lcm(P_{G_1}, P_{G_2}) \cdot N_{G_1}}{P_{G_1}}$;
        
        \item  $\forall i \in \Lset{G_2}$ let
\[A(i) =  \Delta_{k}^{l} \left( \left\{ j \in \Lset{G_1} | \emptyset \neq G_1(j) \subseteq G_2(i) \right\} \right).
\]
Then
\[
\Lcap{G'}=\bigcup_{i \in \Lcap{G_2}^{P_{G'}}} \left\{a \in A(i) | a \in \Lcap{G_1}^{P_{G'}}\right\};
\]

        \item $\forall i \in \Lbar{G'} \; G'(i) = G_1(i)$.
\end{enumerate}
\end{prop}


\begin{example}
  Figure~\ref{fig:SelectDownOperation} shows an example of the \textsl{Select-down} operation in which granularity $G'$ is defined as: $G' = \textsl{Select-down}_2^1(G_1, G_2)$. Since $P_{G_1}=4$, $N_{G_1}=2$ and $P_{G_2}=6$ then $P_{G'}=12$ and $N_{G'}=6$.
  Moreover, since $\Lbar{G_2}=\{-3\}$ and $0 \in \down{-3}^{G_2}$, then $\Lcap{G_2}=\{-3, -2\}$ and $\Lcap{G_2}^{P_{G'}}=\{-3, -2, -1\}$. Intuitively, $A(-3)=\{-5\}$, $A(-2)=\{-2\}$ and $A(-1)=\{1\}$.
  Hence $\Lcap{G'}=\{-5, -2, 1\}$ and therefore, since $0 \in \down{-5}^{G'}$, $\Lbar{G'}=\{-5, -2\}$.
  Finally $G'(-5) = G_1(-5) = \bot(0) \cup \bot(1)$ and $G'(-2) = G_1(-2) = \bot(6)$.
\end{example}

\begin{figure}[ht]
        \centering
                \includegraphics[width=.8\columnwidth]{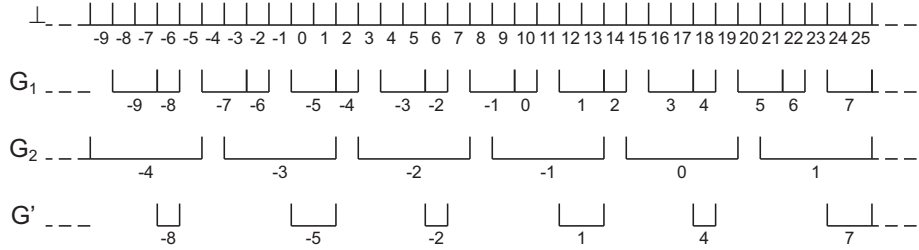}
        \caption{\textsl{Select-down} operation example}
        \label{fig:SelectDownOperation}
\end{figure}

\subsubsection{The Select-up Operation}
\begin{prop}
\label{pr:selectUpOperation}
Given $G'=\textsl{Select-up}(G_1, G_2)$, then:
\begin{enumerate}
        \item $P_{G'}=lcm (P_{G_1}, P_{G_2})$ and $N_{G'}=\frac{lcm(P_{G_1}, P_{G_2}) \cdot N_{G_1}}{P_{G_1}}$;
        
        \item  \[
\Lcap{G'}=\{i \in \Lcap{G_1}^{P_{G'}} | \exists j \in \Lset{G_2} \, s.t. \, \emptyset \neq G_2(j) \subseteq G_1(i) \};
\]

        \item $\forall i \in \Lbar{G'} \; G'(i) = G_1(i)$.
\end{enumerate}
\end{prop}

\begin{example}
  Figure~\ref{fig:SelectUpOperation} shows an example of the \textsl{Select-up} operation: $G' = \textsl{Select-up}(G_1, G_2)$. Since $P_{G_1}=6$, $N_{G_1}=3$ and $P_{G_2}=4$ then $P_{G'}=12$ and $N_{G'}=6$.
  Moreover, since $\Lbar{G_1}=\{-3, -2, -1\}$ and $0 \in \down{-3}^{G_2}$, then $\Lcap{G_1}=\{-3, -2, -1, 0\}$ and $\Lcap{G_1}^{P_G'}=\{-3, -2, -1, 0, 1, 2, 3\}$. Since $G_1(-3) \supseteq G_2(-6)$, $G_1(-1) \supseteq G_2(-4)$ and $G_1(3) \supseteq G_2(0)$ then $\Lcap{G'}=\{-3, -1, 3\}$ and, since $0 \in \down{-3}^{G'}$, then $\Lbar{G'}=\{-3, 1\}$
  Finally $G'(-3) = G_1(-3) = \bot(0) \cup \bot(1)$ and $G'(-1) = G_1(-1) = \bot(4)$.
\end{example}

\begin{figure}[ht]
        \centering
                \includegraphics[width=0.7\columnwidth]{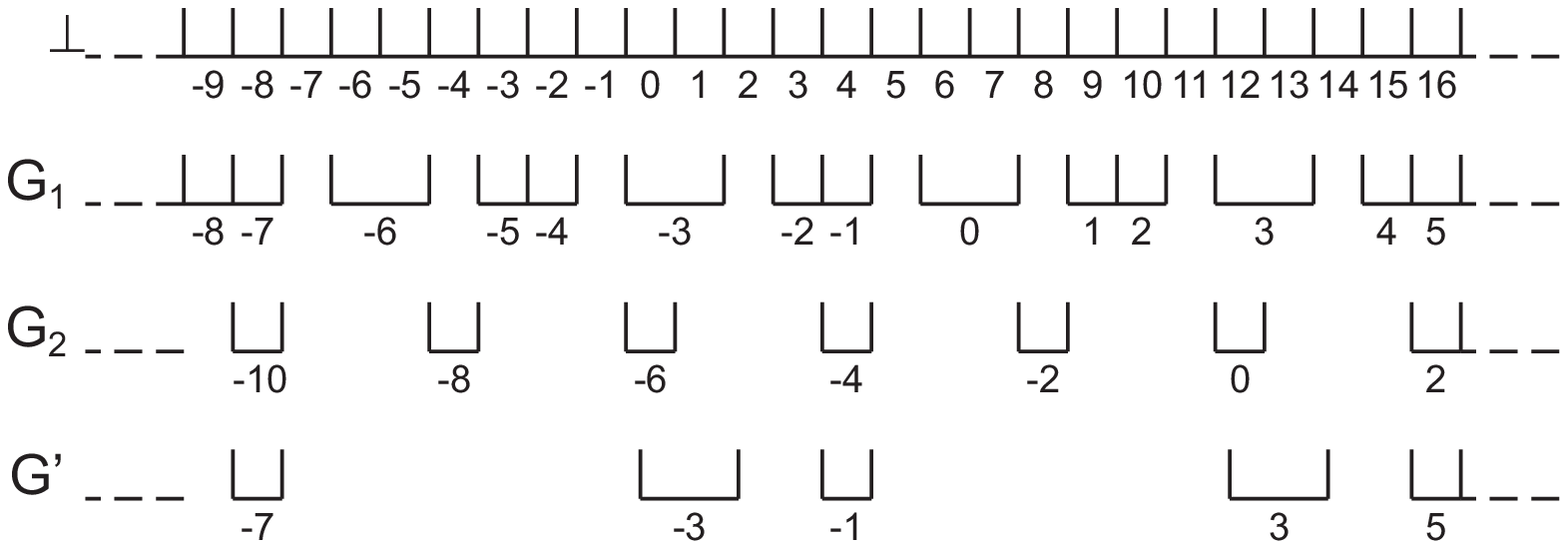}
        \caption{\textsl{Select-up} operation example}
        \label{fig:SelectUpOperation}
\end{figure}

\subsubsection{The Select-by-intersect Operation}
\begin{prop}
\label{pr:selectByIntersectOperation}
Given $G'=\textsl{Select-by-intersect}_k^l(G_1, G_2)$, then:
\begin{enumerate}
        \item $P_{G'}=lcm (P_{G_1}, P_{G_2})$ and $N_{G'}=\frac{lcm(P_{G_1}, P_{G_2})N_{G_1}}{P_{G_1}}$;
        
        \item  then $\forall i \in \Lset{G_2}$ let 
        \[A(i) =  \Delta_{k}^{l} \left( \left\{ j \in \Lset{G_1} | G_1(j) \cap G_2(i) \neq \emptyset \right\} \right).\]
then
\[
\Lcap{G'}=\bigcup_{i \in \Lcap{G_2}^{P_{G'}}} \left\{a \in A(i) | a \in \Lcap{G_1}^{P_{G'}}\right\}.
\]

        \item $\forall i \in \Lbar{G'} \; G'(i) = G_1(i)$.
\end{enumerate}
\end{prop}

\begin{example}
  Figure~\ref{fig:SelectInterOperation} shows an example of the \textsl{Select-by-intersect} operation in which $G' = \textsl{Select-by-intersect}_2^1(G_1, G_2)$. Since $P_{G_1}=4$, $N_{G_1}=2$ and $P_{G_2}=6$ then $P_{G'}=12$ and $N_{G'}=6$.
  Moreover, since $\Lbar{G_2}=\{-3\}$ and $0 \in \down{-3}^{G_2}$, then $\Lcap{G_2}=\{-3, -2\}$ and $\Lcap{G_2}^{P_{G'}}=\{-3, -2, -1\}$. Intuitively, $A(-3)=\{-6\}$, $A(-2)=\{-2\}$ and $A(-1)=\{0\}$.
  Hence $\Lcap{G'}=\{-2, 0\}$ and therefore, since $0 \notin \down{-5}^{G'}$, then $\Lbar{G'}=\{ -2, 0\}$.
  Finally $G'(-2) = G_1(-2) = \bot(6)$ and $G'(0) = G_1(0) = \bot(10)$.
\end{example}

\begin{figure}[ht]
        \centering
                \includegraphics[width=.9\columnwidth]{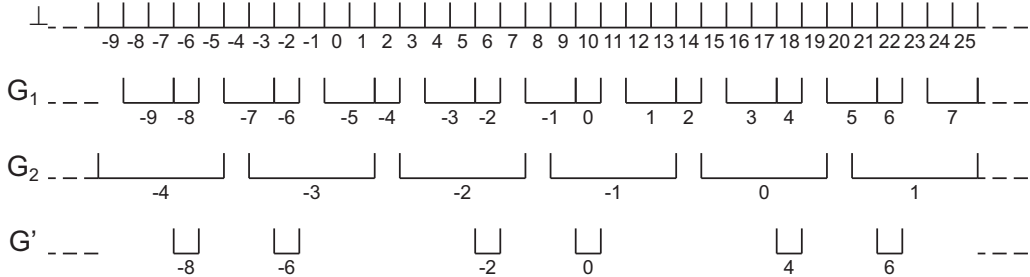}
        \caption{\textsl{Select-by-intersect} operation example}
        \label{fig:SelectInterOperation}
\end{figure}

\subsection{The Set Operations}
\label{sub:setOp}
Since a set operation is valid if the granularities used as argument are
both labeled aligned granularity of another granularity, the
following property is used.

\begin{prop}
\label{prop:SetOperations1}
If $G$ is a labeled aligned subgranularity of $H$, then
$\frac{N_G}{P_G}=\frac{N_H}{P_H}$.
\end{prop}

\begin{prop}
\label{pr:setOperation}
Given $G'=G_1 \cup G_2$, $G''=G_1 \cap G_2$ and $G'''=G_1 \setminus G_2$, then:
\begin{enumerate}
        \item $P_{G'}=P_{G''}=P_{G'''}=lcm(P_{G_1}, P_{G_2})$ and \\ $N_{G'}=N_{G''}=N_{G'''}=\frac{lcm(P_{G_1}, P_{G_2})N_{G_1}}{P_{G_1}} = \frac{lcm(P_{G_1}, P_{G_2})N_{G_2}}{P_{G_2}}$;
        
        \item
         $\Lcap{G'}= \Lcap{G_1}^{P_{G'}} \cup \Lcap{G_2}^{P_{G'}}$; 
         $\Lcap{G''}= \Lcap{G_1}^{P_{G''}} \cap \Lcap{G_2}^{P_{G''}}$;
         $\Lcap{G'''}= \Lcap{G_1}^{P_{G'''}} \setminus \Lcap{G_2}^{P_{G'''}}$;

        \item
                $\forall i \in \Lbar{G'} \; G'(i)=\left \{ \begin{array}{ll}
      G_1(i), & i \in \Lset{G_1}\\
      G_2(i), & \mbox{otherwise},
      \end{array}
      \right.$
      
      $\forall i \in \Lbar{G''} \; G''(i)=G_1(i)$ and $\forall i \in \Lbar{G'''} \; G'''(i)=G_1(i)$
\end{enumerate}
\end{prop}

\begin{example}
  Figure~\ref{fig:SetOperation} shows an example of the set operations. Note that both $G_1$ and $G_2$ are labeled aligned subgranularities of $H$. Then $G' = G_1 \cup G_2$, $G'' = G_1 \cap G_2$ and $G''' = G_1 \setminus G_2$. Since $P_{G_1}=P_{G_2}=6$ and $N_{G_1}=N_{G_2}=6$ then $P_{G'}=P_{G''}=P_{G'''}=6$ and $N_{G'}=N_{G''}=N_{G'''}=2$.
  Moreover, since $\Lcap{G_1}=\{1,2\}$ and $\Lcap{G_2}=\{2,3\}$, then $\Lcap{G'}=\{1,2,3\}$, $\Lcap{G''}=\{2\}$ and $\Lcap{G'''}=\{1\}$.
  Finally $G'(1) = G_1(1)$, $G'(2)=G_1(2)$ and $G'(3) = G_2(3)$; $G''(2) = G_1(2)$ and $G'''(1)=G_1(1)$.
\end{example}

\begin{figure}[ht]
        \centering
                \includegraphics[width=.9\columnwidth]{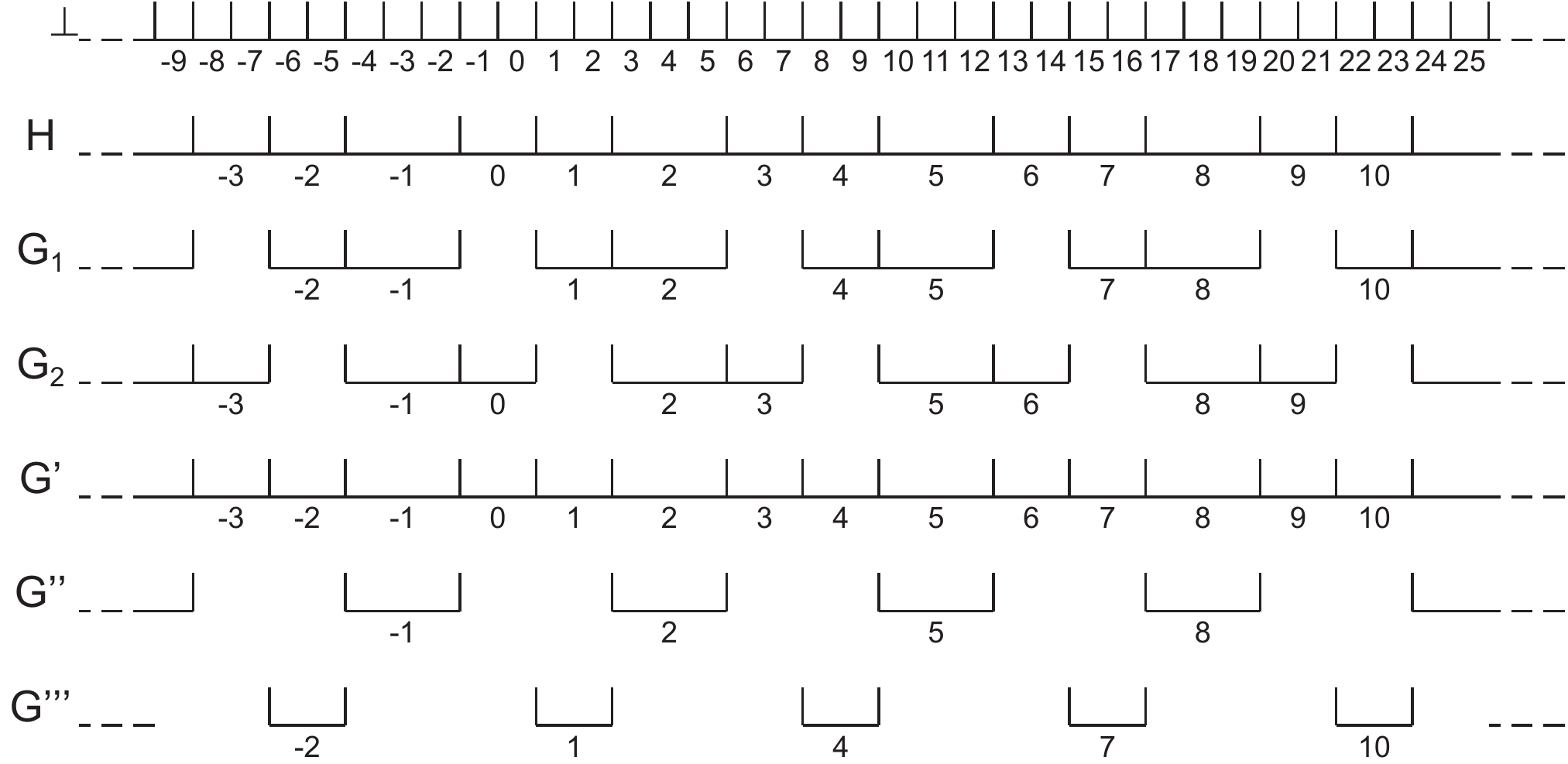}
        \caption{\textsl{Set} operations example}
        \label{fig:SetOperation}
\end{figure}

\subsection{Relabeling}
\label{sec:relabeling}
Granularity processing algorithms are much simpler if restricted to
operate on full-integer labeled granularities. Moreover, a further
simplification is obtained by using only the positive integers as the
set of labels (i.e., $\Lset{}=\mathbb{Z}^{+}$).

In this section we show how to relabel a granularity $G$ to obtain a
full-integer labeled granularity $G'$. A granularity $G''$ such that
$\Lset{G''} = \mathbb{Z}^{+}$ can be obtained by using $G'' =
\textsl{Subset}_1^{\infty}(G')$

Note that with the relabeling process some information is lost: for
example, if $G$ is a labeled aligned subgranularity of $H$ and $G \ne
H$, then, after the relabeling, $G$ is not a labeled aligned
subgranularity of $H$. The lost information is semantically meaningful
in the calendar algebra, and therefore the relabeling must be
performed only when the granularity will not be used as an operator in
an algebraic operation.

Let $G$ be a labeled granularity, $i$ and $j$ integers with $i \in
\Lset{G}$ s.t. $G(i) \neq \emptyset$. The relabeling operation
$\textsl{Relabel}_i^j(G)$ generates a full-integer labeled granularity
$G'$ by relabeling $G(i)$ as $G'(j)$ and relabel the next (and
previous) granule of $G$ by the next (and previous, respectively)
integer.  More formally, for each integer $k$, if $k=j$, then let
$G'(k) = G(i)$, and otherwise let $G'(k) = G(i')$ where $G(i')$ is the
$|j-k|$-th granule of $G$ after (before, respectively) $G(i)$. If the
required $|j-k|$-th granule of $G$ does not exist, then let
$G'(k)=\emptyset$. Note the $G'$ is always a full-integer labeled
granularity.

The relabeling procedure can be implemented in the periodic representation we adopted by computing the value of $l_{G'}$. It is easily seen that once $l_{G'}$ is known, the full characterization of $G'$ can be obtained with: $P_{G'}=P_{G}$; $N_{G'}=R_{G'}=R_{G}$ and $\Lbar{G'} = \{l_{G'}, l_{G'} + 1, \ldots, l_{G'} + N_{G'} - 2, l_{G'} + N_{G'} - 1\}$. It is clear that the explicit representation of the granules is not modified.

To compute $l_{G'}$ consider the label $i' = i - \left\lfloor \frac{i-l_G}{N_G}\right\rfloor \cdot N_G$;
$i'$ represents the label of $\Lbar{G}$ such that $i - i'$ is a multiple of $N_G$. Therefore it is clear that the label $j' \in \Lbar{G'}$ s.t. $G'(j') = G(i')$ can be computed by $j' = j - \left\lfloor \frac{i-l_G}{N_G}\right\rfloor \cdot N_{G'}$.
Finally $l_{G'}$ is obtained with $l_{G'}=j' - |\delta|$ where $\delta$ is the distance, in terms of number of granules of $G$, from $G(l_G)$ to $G(i')$.

\begin{example}
Figure~\ref{fig:relabeling} shows an example of the \textsl{Relabel} operation: $G' = \textsl{Relabel}_{33}^{4}(G)$. Since $P_G=4$ and $R_G=2$ then $P_{G'}=4$ and $N_{G'}=2$.
Moreover, $i'=33 - \left\lfloor \frac{33-6}{5} \right\rfloor \cdot 5 = 8$ and $j' = 4 - \left\lfloor \frac{33-6}{5}\right\rfloor \cdot 2 = -6$.
Since $l_G = 6$ and $i'=8$ then $G(i')$ is the next granule of $G$ after $G(l_G)$. Then $\delta=1$ and  hence $l_{G'}=-6 - 1 = -7$. It follows that $\Lbar{G'}=\{-7, -6\}$. Finally $G'(-7) = G(6)$ and $G'(-6)=G(8)$.
\end{example}

\begin{figure}[ht]
        \centering
                \includegraphics[width=.8\columnwidth]{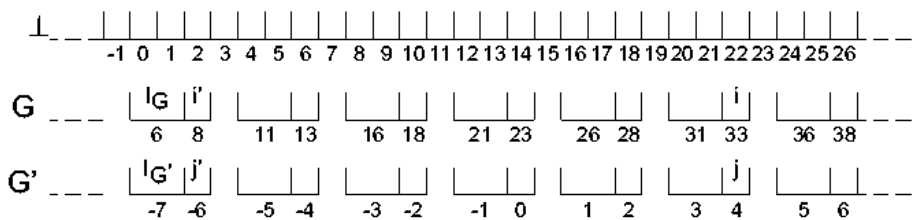}
        \caption{\textsl{Relabeling} example}
        \label{fig:relabeling}
\end{figure}

The GSTP constraint solver imposes that the first non-empty granule of any granularity ($\bot$ included) is labeled with $1$. Therefore, when using the relabeling operation for producing granularities for GSTP, the parameter $j$ must be set to $1$. The parameter $i$ has to be equal to the smallest label among those that identify granules of $G$ covering granules of $\bot$ that are all labeled with positive values. By definition of $l_G$, $i=l_G$ if $min(\down{l_G}^G) > 0$; otherwise $i$ is the next label of $G$ after $l_G$.

\subsection{Complexity Issues}
\label{sec:Complexity}
For each operation the time necessary to perform the three conversion
steps, depends on the operation parameters (e.g. the ``grouping
factor'' $m$, in the \textsl{Group} operation) and on the operand
granularities (in particular the period length, the period label
distance and the number of granules in one period).

A central issue is that if an operand granularity is not the bottom
granularity, then its period is a function of the periods of the
granularities that are the operands in the operation that defines
it. For most of the algebraic operations, in the worst case the period
of the resulting granularity is the product of the periods of the
operands granularity.

For all operations, the \textbf{first step} in the conversion process
can be performed in a constant or logarithmic time. Indeed the
formulas necessary to derive the period length and the period label
distance involve (i) standard arithmetic operations, (ii) the
computation of the Greatest Common Divisor and (iii) the computation
of the least common multiple. Part (i) can be computed in a constant
time while (ii) and (iii) can be computed in a logarithmic time using
Euclid's algorithm.

For some operations, the \textbf{second step} can be performed in
constant time (e.g. \textsl{Group}, \textsl{Shift} or
\textsl{Anchored-group}) or in linear time (e.g. set operations). For
the other operations it is necessary to compute the set $S$ of labels
of a granularity $G$ such that $\forall i \in S \; G(i) \subseteq
H(j)$ where $H$ is a granularity and $j \in \Lset{H}$ (analogously if
$S$ is the set such that $\forall i \in S \; G(i) \supseteq H(j)$ or
$\forall i \in S \; (G(i) \cap H(j) \neq \emptyset)$). This
computation needs to be performed once for each granule $i \in
P_H^{P_{G'}}$.  The idea of the algorithm for solving the problem has
been presented in Section~\ref{sec:computability}.  Several
optimizations can be applied to that algorithm, but in the worst case
(when $H$ covers the entire time domain) it is necessary to perform a
number of $\up{\cdot}^G$ operations linear in the period length of the
resulting granularity.
If an optimized data structure is used to represent the granularities,
the $\up{\cdot}^G$ operation can be performed in constant time
\footnote{If a non-optimized data structure is used, $\up{\cdot}^G$
  requires logarithmic time.}, then the time necessary to perform the
second step is linear in the period length of the resulting
granularity ($O(P_{G'})$).

The \textbf{last step} in the conversion process is performed in
linear time with respect to the number of granules in a period of
$G'$.
%

The complexity analysis of the conversion of a general algebraic
expression needs to consider the composition of the operations and
hence their complexity. Finally, relabeling, can be done in linear
time.

A more detailed complexity analysis is out of the scope of this work.

\section{Minimal Representation and Experimental Results}
\label{sec:minimality}
In this section we address the problem of guaranteeing that the
converted representation is minimal in terms 
of
the period length.
As we will show in Example~\ref{ex:non-minimal} the conversion formulas proposed
in this paper do not guarantee a minimal representation of the result
and it is not clear if conversion formulas ensuring minimality exist.
Our approach is to apply a minimization step in the conversion.

The practical applicability of the minimization step depends on the
period length of the representation that is to be minimized.
Indeed, in our tests we noted that the minimization step is efficient
if the conversion formulas proposed in Section \ref{ch:CalAlg2PSet} are adopted,
while it is impractical when the conversion procedure returns a period
that is orders of magnitude higher than the minimal one
as would be the case if conversion formulas were constructed in a
naive way.

%


\subsection{Period Length Minimization}
\label{sub:minimization}
As stated in Section~\ref{ch:gran}, each granularity can have
different periodical representations and, for a given granularity, it
is possible to identify a set of representations that are
\emph{minimal} i.e. adopting the smallest period length.

Unfortunately, the conversions do not always return a minimal representation, as shown by Example~\ref{ex:non-minimal}.

\begin{example}\label{ex:non-minimal}
  
  Consider a calendar that has \texttt{day} as the bottom granularity.
  We can define $\texttt{week}$ as
  $\texttt{week}=\textsl{Group}_{7}(\texttt{day})$; by applying the
  formulas for the \textsl{Group} operation we obtain
  $P_{\texttt{week}}=7$ and $N_{\texttt{week}}=1$.
  
  We can now apply the \textsl{Altering-tick} operation to add one day
  to every first week every two weeks. Let this granularity be $G_1 =
  \textsl{Alter}_{1, 1}^{2}(\texttt{day}, \texttt{week})$; applying
  the formulas for the \textsl{Altering-tick} operation we obtain
  $P_{G_1}=15$ and $N_{G_1}=2$.
  
  We can again apply the \textsl{Altering-tick} operation to create a
  granularity $G_2$ by removing one day from every first granule of
  $G_1$ every two granules of $G_1$: $G_2 = \textsl{Alter}_{1, -1}^{2}(\texttt{day}, G_1)$. Intuitively, by applying this operation
  we should get back to the granularity \texttt{week}, however using
  the formulas for the \textsl{Altering-tick} operation we obtain
  $P_{G_2}=14$ and $N_{G_2}=2$; Hence $G_2$ is not minimal.
\end{example}

In order to qualitatively evaluate how close to the minimal
representations the results of our conversions are, we performed a set
of tests using an algorithm \cite{time05} for minimality checking. In our
experimental results the conversions of algebraic expressions
defining granularities in real-world calendars, including many
user-defined non-standard ones, always returned exactly minimal
representations. Non-minimal ones could only be obtained by
artificial examples like the one presented in
Example~\ref{ex:non-minimal}.

Although a non-minimal result is unlikely in practical calendars, the
minimality of the granularity representation is known to greatly
affect the performance of the algorithms for granularity processing,
e.g., granularity constraint processing \cite{BWJ:aij02}, calendar
calculations \cite{TauZaman-SPE07}, workflow temporal support
\cite{CombiER03}.  Hence, we considered an extension of the
conversion algorithm by adding a minimization step exploiting
the technique illustrated by Bettini et al. \citeyear{time05}
to derive a minimal representation.

The choice of using only the conversion algorithm or the extended one with minimizations, should probably be driven by performance considerations.
In Section~\ref{sub:ExperimentalResults} we report the results of our experiments
showing that generally it is advantageous to apply the minimization step.
In our implementation, presented in Section~\ref{sub:implementation}, it is possible to specify if the minimization step should be performed.

\subsection{Implementation of the \emph{CalendarConverter} Web Service}
\label{sub:implementation}

The conversion formulas presented in Section~\ref{ch:CalAlg2PSet} have
been implemented into the \emph{CalendarConverter} web service that
converts Calendar Algebra representations into the equivalent
periodical ones.  More precisely, given a calendar in which
granularities are expressed by Calendar Algebra operations, the
service converts each operation into an equivalent periodical
representation.
%
%

The service first rewrites each calendar algebra expression in order
to express it only in terms of the bottom granularity. For example, if
the bottom granularity is $\texttt{hour}$, the expression
$\texttt{Monday}=\textsl{Select-down}^1_1(\texttt{day},
\texttt{week})$ is changed to
\[\texttt{Monday}=\textsl{Select-down}^1_1(\textsl{Group}_{24}(\texttt{hour}), \textsl{Group}_7(\textsl{Group}_{24}(\texttt{hour})))\]
Then, Procedure~\ref{alg:convert} is run for each granularity's expression.
The idea is that the periodical representation of each subexpression
is recursively computed starting from the expressions having the
bottom granularity as operand.
Once each operand of a given
operation has been converted to periodical representation, 
the corresponding formula presented in Section~\ref{ch:CalAlg2PSet} is
applied. We call this step the \emph{ConvertOperation} procedure.
%
%

A trivial optimization of Procedure~\ref{alg:convert} consists in
caching the results of the conversions of each 
subexpression
so that it
is computed only once, even if the 
subexpression
appears several times (like $\textsl{Group}_{24}(\texttt{hour})$ in the above \texttt{Monday} definition).



\begin{algorithm}[ht]
\floatname{algorithm}{Procedure}
\caption{ConvertExpression}
\label{alg:convert}
\begin{itemize}
      \item \textbf{Input}: a calendar algebra expression $ex$; a boolean value $minimize$ that is set to \textbf{true} if the minimization step is to be executed;
      \item \textbf{Output}: the periodical representation of $ex$;
      \item \textbf{Method}:
\end{itemize}
\begin{algorithmic}[1]
\IF{($ex$ is the bottom granularity)}
        \STATE \textbf{return} the periodical representation of the bottom granularity
\ENDIF
\STATE $operands:=\emptyset$
\FOR{(each operand $op$ of $ex$)}
      \STATE add ConvertExpression($op$, $minimize$) to $operands$;
\ENDFOR
\STATE $result := $ConvertOperation(ex.getOperator(), operands)
\IF{($minimize$)}
        \STATE minimize the periodical representation of $result$
\ENDIF
\STATE \textbf{return} result;
\end{algorithmic}
\end{algorithm}

\subsection{Experimental Results}
\label{sub:ExperimentalResults}
Our experiments address two main issues: first, we evaluate how the conversion formulas impact on the practical applicability of the conversion procedure and, second, we evaluate how useful is the minimization step.

\vspace{2mm}

For the first issue, we execute the conversion procedure with two different sets of conversion formulas and compare the results. The first set is laid out in Section~\ref{ch:CalAlg2PSet}.
The other, that is less optimized, is taken from the preliminary version of this paper~\cite{BMW:time04}.

\begin{table}
\caption{Impact of the conversion formulas on the performance of the
  conversion and minimization procedures (time in milliseconds).}
\label{table}
\label{tab:optimizedFormulas}
\begin{center}
\begin{tabular}{||l|l||r|r|r||r|r|r||} \hline
\multicolumn{2}{||c||}{Calendar} &
\multicolumn{3}{c||}{Section~\ref{ch:CalAlg2PSet} formulas} &
\multicolumn{3}{c||}{Less optimized formulas}
\\ \hline
Period & Bot &  Conv. & Min. & Tot. & Conv. & Min. & Tot.  \\ \hline
1 year & day & 4 & 2 & 6 & 62 & 32 & 94 \\ \hline
4 years & day & 7 & 2 & 9 & 76 & 55 & 131 \\ \hline
1 year & hour & 9 & 2 & 11 & 2,244 & 126,904 & 129,148 \\ \hline
4 years & hour & 16 & 4 & 20 & 4,362 & 908,504 & 912,866 \\ \hline
100 years & day & 127 & 9 & 136 & 3,764 & 1,434,524 & 1,438,288 \\ \hline
\end{tabular}
\end{center}
\end{table}

Table~\ref{table} shows that when converting calendars having
granularities with small minimal period length (first two
rows),
using the formulas in Section~\ref{ch:CalAlg2PSet} improves the performance by one order of magnitude; However, conversions and minimizations are almost instantaneous with both approaches.
On the contrary, when the minimal period length is higher, (last three
rows)
%
the time required to minimize the periodical representation is up to five orders of magnitude larger if the formulas proposed by Bettini et al. \citeyear{BMW:time04} are used; as a consequence, the entire conversion may require several minutes while, using the formulas presented in Section~\ref{ch:CalAlg2PSet}, it still requires only a fraction of a second.
If the period length is even larger, the conversion procedure is impractical if the formulas presented by Bettini et al. \citeyear{BMW:time04} are used, and indeed in our experiments we did not obtain a result in less than thirteen hours.

\vspace{2mm}

For the second issue, we perform a set of three experiments.
In the first one we compare the performance of the conversion procedure with the performance of the minimization step.
In the experiment we consider the case in which the conversion procedure produces minimal representations. In this case the minimization step is always an overhead since it cannot improve the performance of the conversion procedure.

Figure~\ref{fig:perf1} shows the result of the experiment. Four calendars are considered, each one containing a set of granularities of the Gregorian calendar.
The four calendars differs in the values of two parameters: the bottom granularity (it is \texttt{second} for cal-1 and cal-3 while it is \texttt{minute} for cal-2 and cal-4) and the period in which leap years and leap years exceptions are represented (it is $1$, $4$, $100$ and $400$ years for cal-1, cal-3, cal-2 and cal-4 respectively);
As a consequence, the minimal period length of the granularities \texttt{month} and \texttt{year} is about $3 \cdot 10^7$ for cal-1, $5 \cdot 10^7$ for cal-2, $10 ^ 8$ for cal-3 and $2 \cdot 10^8$ for cal-4. 

\begin{figure}[htbp]
  \centering
\includegraphics[width=.7\textwidth]{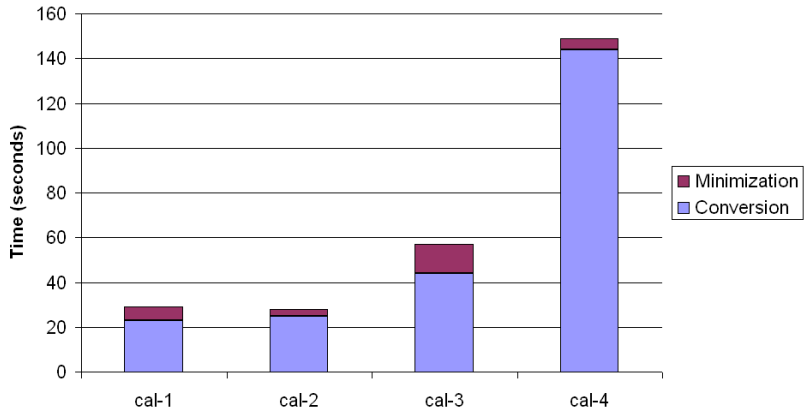}
\caption{Impact of minimization over conversion; minimal conversions case.}
  \label{fig:perf1}
\end{figure}

As can be observed in Figure~\ref{fig:perf1}, the ratio between the time required to perform the  conversions and the time required for the minimization step varies significantly from a minimum of $3\%$ for cal-4 to a maximum of $23\%$ for cal-3.
The reason is that the complexity of the conversion procedure is mainly affected by the period length of the granularity having the largest period length. On the other hand, the complexity of the minimization step is affected also by other features of the granularities such as their internal structure and the number of integers that can divide at the same time the period label distance, the period length and the number of granules in one period; For more details see \cite{time05}.

\vspace{0.1cm}

In the second experiment we consider the case in which the conversion procedure produces a non-minimal representation for a granularity in the input calendar;
in this case it is possible
to benefit from the minimization step.
For example, suppose that a granularity $G$ is converted and that it is then used as an argument of another Calendar Algebra operation that defines a granularity $H$. The time required to compute the periodical representation of $H$ strongly depends on the period length of $G$;
If the period length of $G$ is reduced by the execution of the minimization step, the conversion of $H$ can be executed faster.

We produced this situation using a technique similar to the one of Example~\ref{ex:non-minimal}; we created Calendar Algebra definitions of the Gregorian calendar in which the granularity \texttt{day} is converted into a granularity having a non-minimal representation.
Figure~\ref{fig:perf2} shows the performance obtained converting the same granularities that were used in Figure~\ref{fig:perf1}. The difference was that in this case the
definition of the granularity \texttt{day} is such that, after the conversion procedure,
its period is twice as large as the minimal one (i.e., $48$ hours or $2880$ minutes or $172800$ seconds depending on the bottom granularity that is used).
It can be easily seen that in this case the use of the minimization step can improve the performance of the entire algorithm. Indeed, when the minimization step is performed, the conversion procedure requires about one half of the time that is required when no minimization is performed.

\begin{figure}[htbp]
  \centering
\includegraphics[width=.7\textwidth]{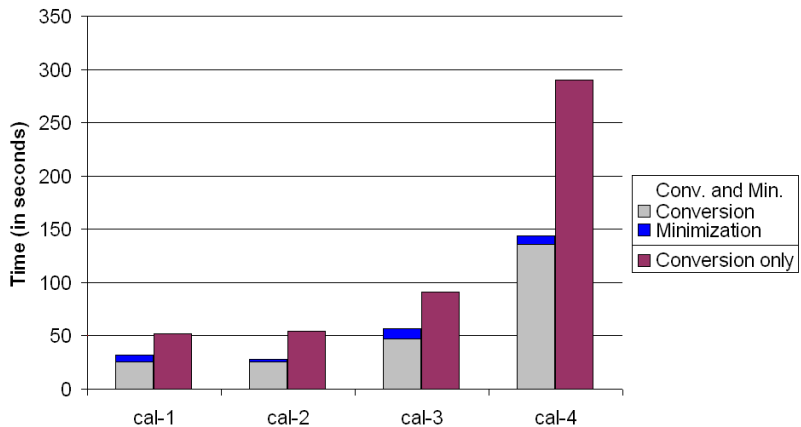}
\caption{Impact of minimization over conversion; non-minimal case.}
  \label{fig:perf2}
\end{figure}

\vspace{0.1cm}

In the third experiment we evaluate the impact of the minimal representation on the performance of applications involving intensive manipulations of granularities.
In the test we use the GSTP solver as such an application;
it computes solutions of temporal constraints with granularities. A description of the architecture of the GSTP system is provided in Section~\ref{sec:gstp}. 

Figure~\ref{fig:perf3} shows our experiments performed on four temporal constraint networks with granularities.
The four networks differs in the number of variables, in the number of constraints and in the granularities used to express the constraints. The networks labeled as ``non-minimal'' use granularities definitions that are obtained with a technique similar to the one used in Example~\ref{ex:non-minimal}, and have a period that is twice as large as the minimal one.

\begin{figure}[htbp]
  \centering
\includegraphics[width=.7\textwidth]{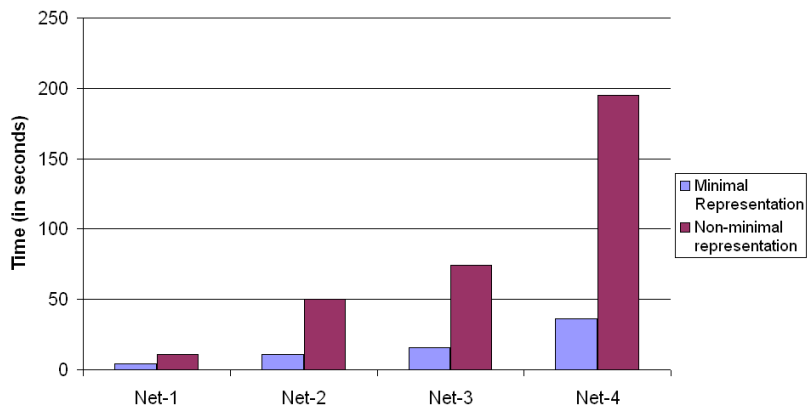}
\caption{Impact of minimal representations on the performance of the GSTP solver.}
  \label{fig:perf3}
\end{figure}

Figure~\ref{fig:perf3} shows that the use of minimal representations 
greatly
improves the performance of the GSTP solver. Indeed in our experiments
the ratio between the time required to solve the network using a
non-minimal representation and a minimal one is between three and
five. Moreover, the 
more
time required to solve the network, the
greater
the improvement obtained using the minimal representation; this means that for very complex temporal networks we expect the improvement to be even higher.

Considering the results of our experiments, we conclude that, in general, it is advisable to perform the minimization step.
In particular, it is very advantageous in the specific case of GSTP, based on the following considerations: 
i) the time required to perform the minimization step is only a fraction of the time required to perform the conversion procedure,
ii) the conversions are performed off-line in most cases, with respect
to granularity processing, and conversion results are cached for
future use,
and
iii) the period length strongly influences the GSTP
processing time that is in most cases much
longer
 than the time needed
for conversion.

\section{Applications}
\label{sec:applications}

In this section we complement the motivations for this work
with a sketch of the applications enabled by the proposed conversion.
Firstly we describe the GSTP system, as an example of 
applications
involving intensive manipulation of time granularities. GSTP is used
to check the consistency and to find solutions of temporal constraint
satisfaction problems with granularities\footnote{For a detailed
  description of the system, see \cite{BMP-LNAI05}.}; It has also
been applied to check the consistency of inter-organizational workflow
models \cite{BettiniWorkflow}.  Then, we discuss the use of Calendar
Algebra to define new granularities that may later be part of the
input of reasoning services, such as GSTP.


\subsection{The GSTP System}
\label{sec:gstp}
The GSTP system has been developed at the University of Milan with the
objective of providing universal access to the implementation of a set
of algorithms for multi-granularity temporal constraint satisfaction
\cite{BWJ:aij02}.
It allows the user to specify binary constraints of the form $Y - X \in [m, n]G$ where $m$ and $n$ are the minimum and maximum values of the distance from $Y$ to $X$ in terms of granularity $G$. Variables take values in the positive integers, and unary constraints can be applied on their domains. 
For example, the constraint: \emph{Event2 should occur 2 to 4 business
  days after the occurrence of Event1} can be modeled by $Occ_{E2} -
Occ_{E1} \in [2, 4]BDay$. This problem 
is
considered
an
extension of STP \cite{DMP-AI91} to multiple and arbitrary granularities.
To our knowledge, GSTP is the only available system to solve this
class of
temporal constraint satisfaction problems.

Figure~\ref{fig:arch} shows the general architecture of the GSTP
system. There are three main modules: the constraint solver; the web
service, which enables external access to the solver; 
and
a user
interface that can be used locally or remotely to design and analyze
constraint networks.

\begin{figure}[htbp]
  \centering
\includegraphics[width=0.7\textwidth]{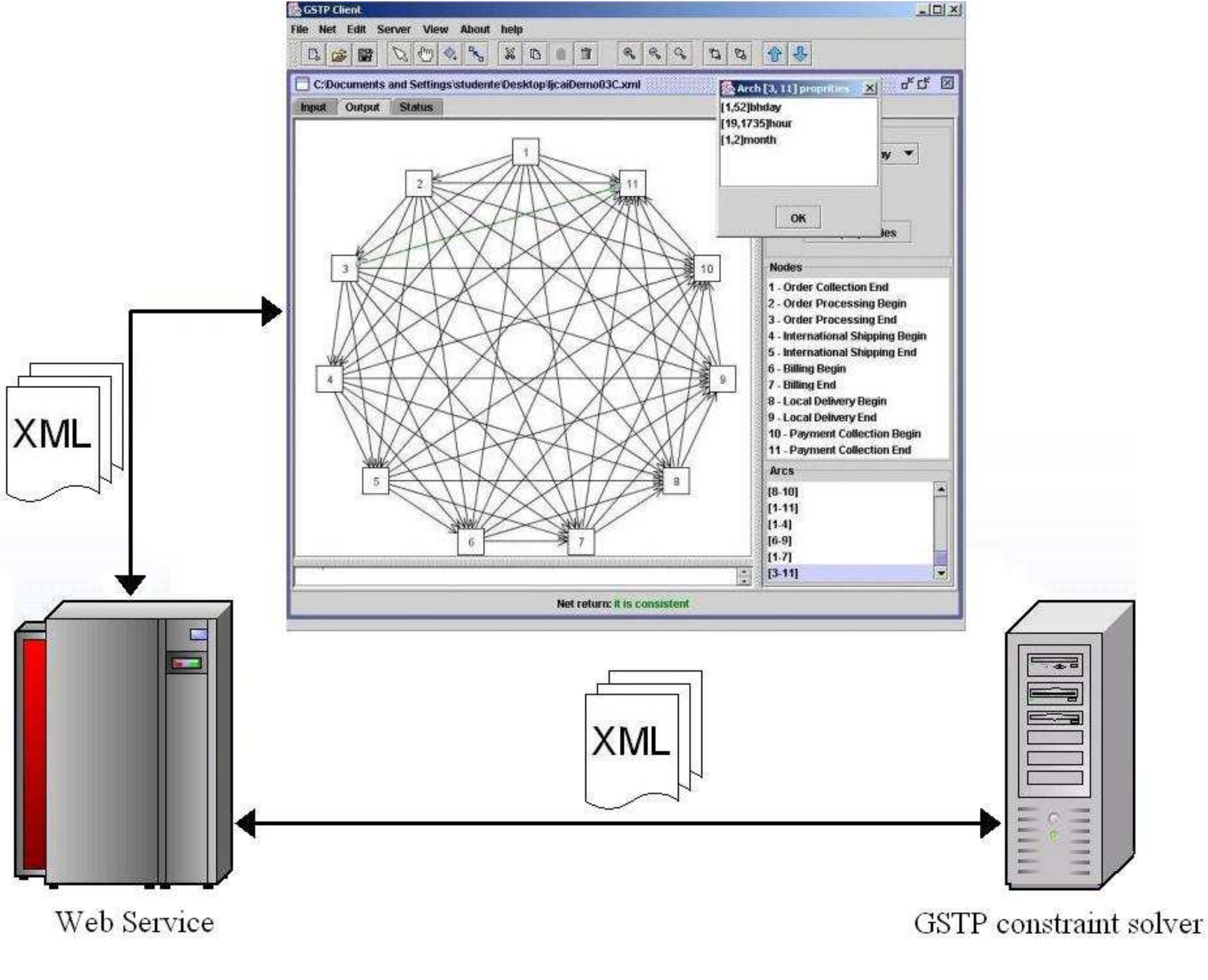}
\caption{The GSTP Architecture}
  \label{fig:arch}
\end{figure}

The constraint solver is the C implementation of the ACG algorithm
which has been proposed by Bettini et al. \citeyear{BWJ:aij02}, and it runs on a server
machine.  Following the approach of Bettini et al. \citeyear{BWJ:aij02}, the solver uses
the representation of granularities based on periodical sets. This
representation makes it possible to efficiently compute the core
operations on granularities that are required to solve the constraint
satisfaction problem. These operations involve, for example, the union
and the intersection of periodical sets.  While we cannot exclude that
these operations may be computed in terms of alternative low level
representations, it seems much harder to obtain similar results if a
high level representation, such as Calendar Algebra, is used.


The second module of the system is the Web Service that defines,
through a WSDL specification, the parameters that can be passed to the
constraint solver, including the XML schema for the constraint network
specification.
%

\begin{figure}[htbp]
  \centering
\includegraphics[width=.7\textwidth]{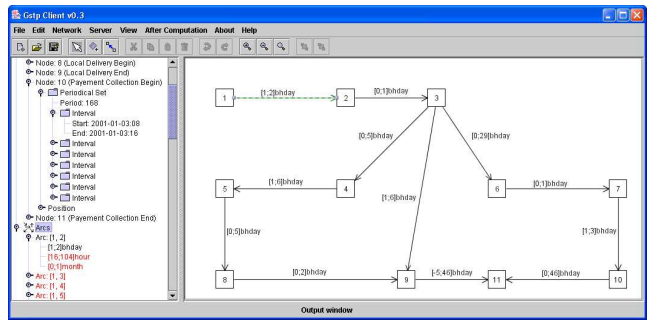}
 \caption{The GSTP User Interface}
  \label{fig:client}
\end{figure}

The third module is a remote Java-based user interface, which allows the user to easily edit 
constraint networks, to submit them to the constraint solver, and to
analyze results. In particular, it is possible to have views in terms
of specific granularities, to visualize implicit constraints, to browse 
descriptions of domains, and to obtain a network solution.
Fig.~\ref{fig:client} shows a screenshot from the interface.


\subsection{Defining New Granularities}
\label{subsec:CalDef}

While the GSTP solver can handle arbitrary granularities, new
granularities must be added by editing their explicit periodical
representation.  This is true in general for any 
multi-granularity reasoning service based on a low-level representation of granularities, and it
is a painful task when the granularities have a large period.
%
%
%
For example, in the experimental results illustrated in Figure~\ref{fig:perf3},
we used a representation of the granularity \texttt{month} that considers leap years and leap years exceptions in a period of 400 years.
In this case, the users have to specify the representation of
4800 granules i.e., the number of months in 400 years.

Because the period length of real world granularities is generally
high, a graphical interface does not help if it only supports the user
to individually select the explicit granules.  An effective solution
requires the use of implicit or explicit operations on granules.
Among the various proposals, Calendar Algebra provides the richest set
of such operators.
%
A question arises: is the definition of granularities in terms of
Calendar Algebra really simpler than the specification of the
periodical representation?  Calendar Algebra does not seem to be user
friendly: the exact semantics of each operator may not be immediate
for an
inexperienced
user
and some time is required in order to learn how to use each operator.
%

In practice, we do not think that it is reasonable to ask an unexperienced user to define granularities by writing Calendar Algebra expressions.
Nevertheless, we do think that Calendar Algebra can be used by specialized user interfaces to guide the user when specifying granularities.
In this sense, we believe that Calendar Algebra plays the same role that SQL does in the definition of databases queries.
Similarly to Calendar Algebra, SQL is an abstraction tool that can be directly exploited in all its expressive power by an advanced user, but can also be used by a less experienced user through a graphical user interface, possibly with a reduced expressiveness.

As mentioned above, in the case of periodical representations, graphical user interfaces are not sufficient for making the specification of new granularities practical.
On the contrary, in the case of Calendar Algebra, user interfaces can strongly enhance the usability of Calendar Algebra, making its practical use possible also for the definition of involved granularities. There are at least two reasons for this difference.
Firstly, the main difficulty of Calendar Algebra is the understanding of the semantics of the operators and the choice of the most appropriate one for a given task.
An effective user interface can hide the existence of the algebraic operators to the user 
showing only how the operators modify existing granularities (i.e., the semantics of the operators).
Secondarily, Calendar Algebra allows the compact definition of granularities.
This is due to the fact that the Calendar Algebra operations are specifically designed to reflect the intuitive ways in which users define new granularities.

Example~\ref{ex:ui} shows how a graphical user interface can be effectively used to define a new granularity in terms of Calendar Algebra expression.

\begin{figure}
                \subfigure[][\label{fig:step1} Step 1.]{\includegraphics[scale=0.30]{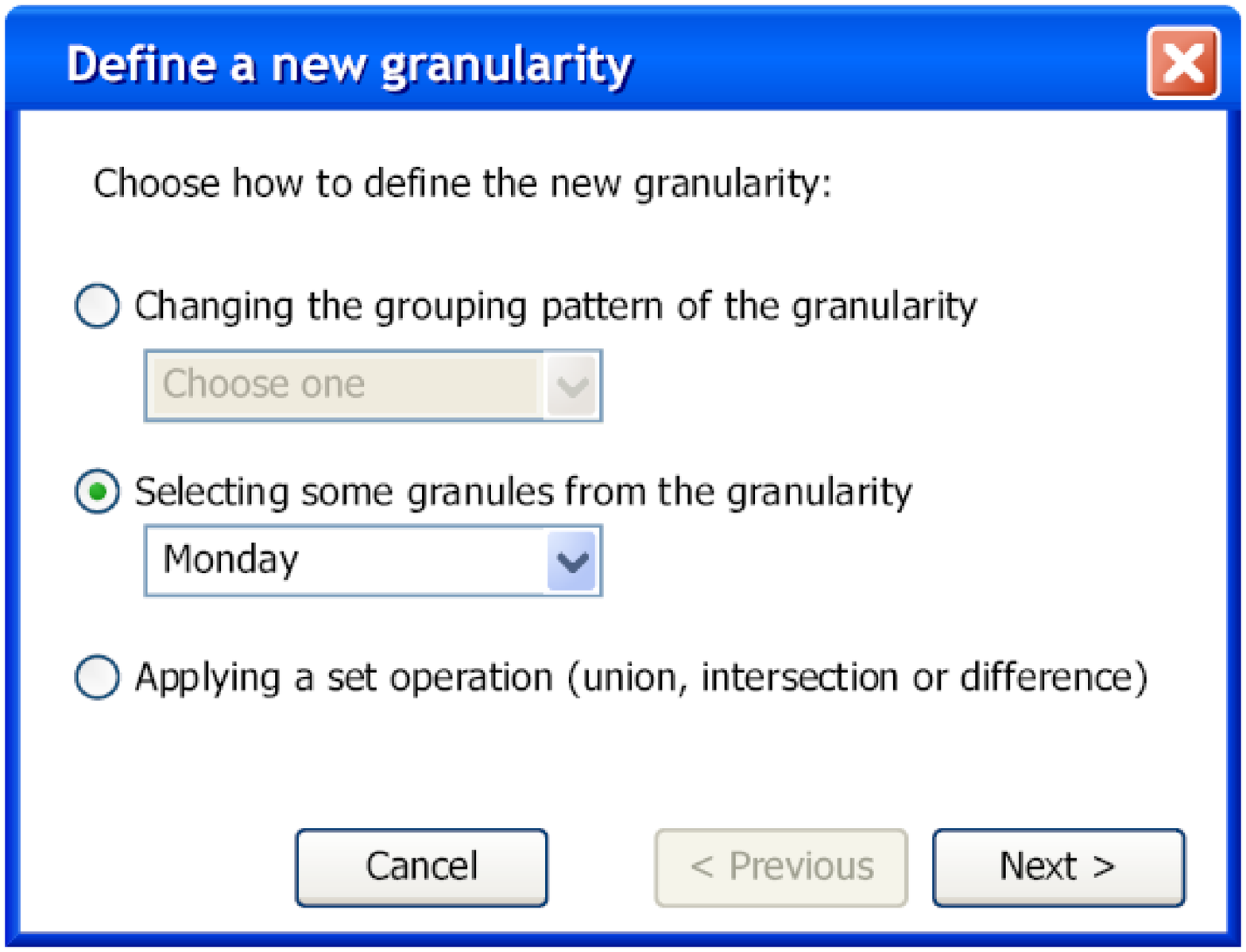}}
                \subfigure[][\label{fig:step2} Step 2.]{\includegraphics[scale=0.30]{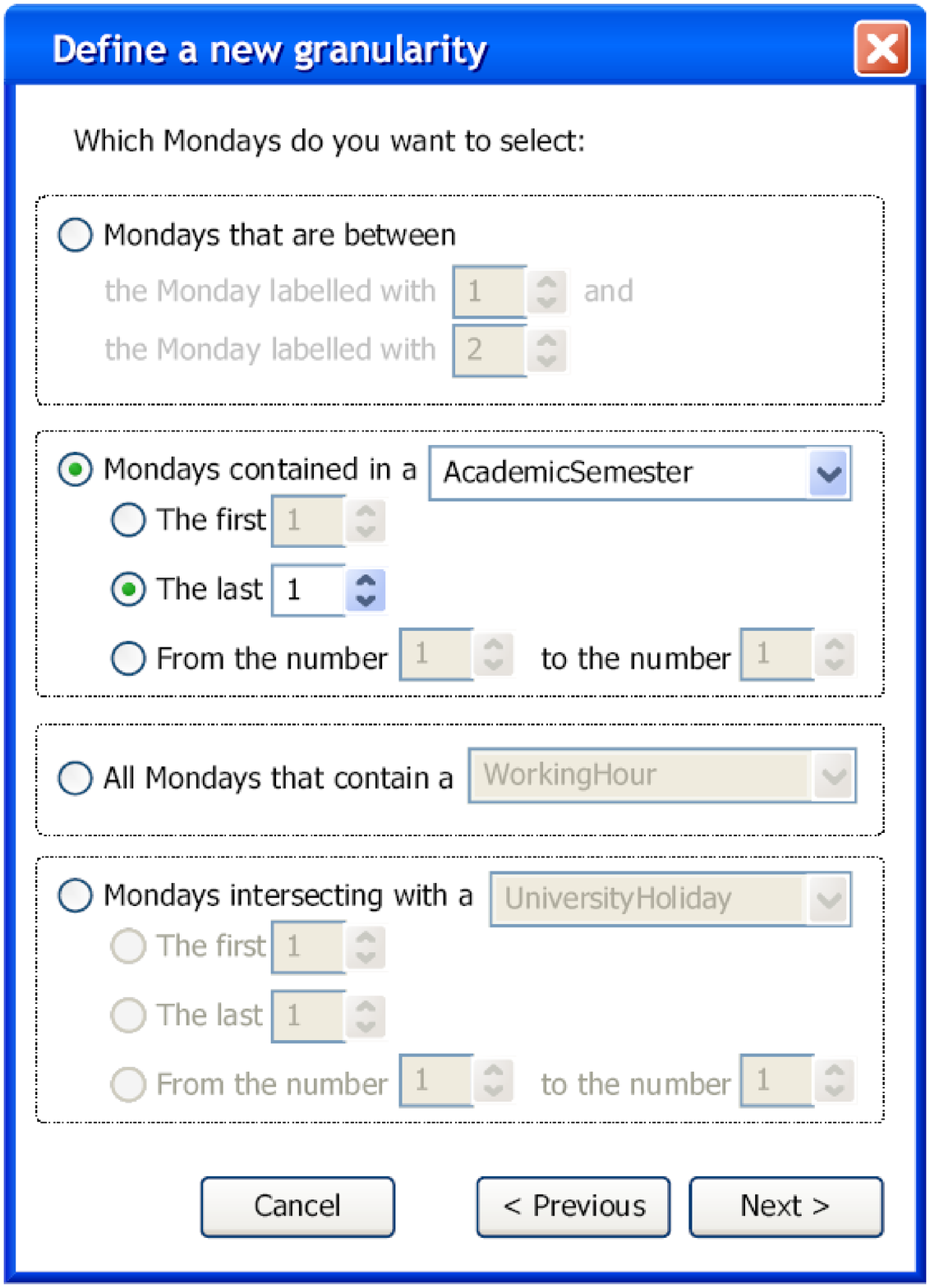}}
                \subfigure[][\label{fig:step3} Step 3.]{\includegraphics[scale=0.30]{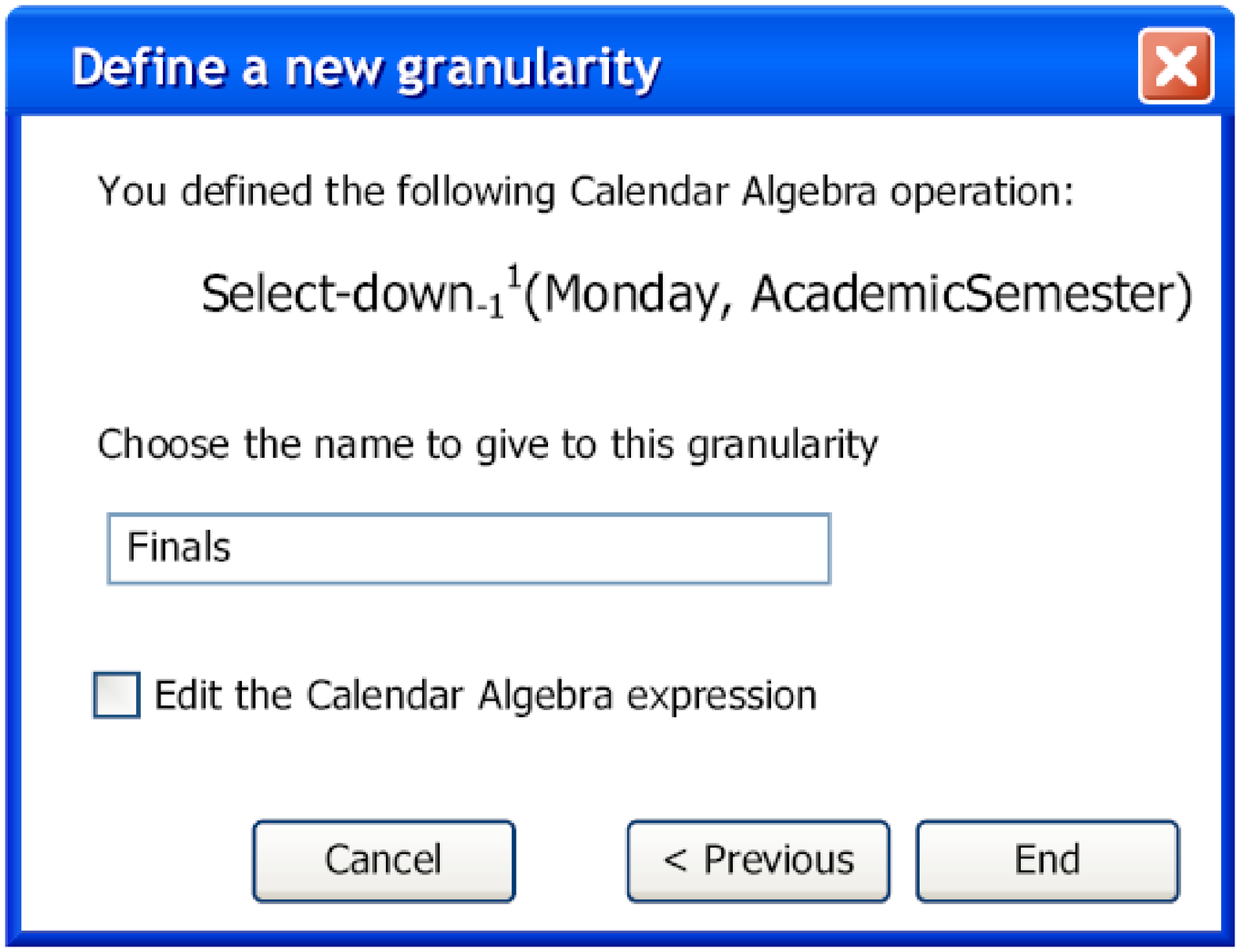}}
                \caption{A 3-steps wizard for visually defining a granularity using Calendar Algebra}
\end{figure}

\begin{example}
\label{ex:ui}
This example shows how a graphical user interface can be used to support the user in the definition of the granularity \texttt{final} as the set of days, each one corresponding to the last Monday of every academic semester.
We assume that the granularities \texttt{Monday} and \texttt{academicSemester} have already been defined.
The graphical user interface that we use in this example is a wizard that guides the user step by step.
In the first step (Figure~\ref{fig:step1}) the user chooses the kind of operation he wants to perform.
In the second step (Figure~\ref{fig:step2}) the user can provide more details about how he wants to modify the operand granularity (\texttt{Monday}, in the example). The results of this choice is a Calendar Algebra expression that is shown in the third step (Figure~\ref{fig:step3}); in this last window the user can also give a name to the granularity that has been defined.

\end{example}



\subsection{The Global Architecture}
\label{subsec:arch}

\begin{figure}[htbp]
  \centering
\includegraphics[width=0.5\textwidth]{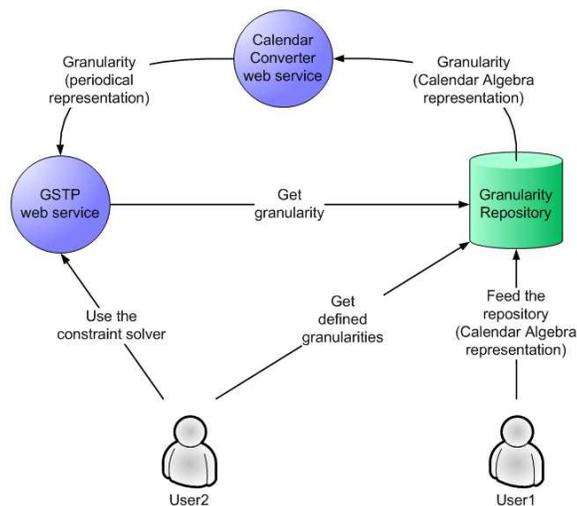}
 \caption{Integration of GSTP and CalendarConverter web services}
  \label{fig:composed-arch}
\end{figure}

Figure~\ref{fig:composed-arch} shows a possible architecture for the
integration of GSTP, the interface for new granularity definitions and
the CalendarConverter web service.  A granularity repository collects
the Calendar Algebra definitions. Upon request by the GSTP system definitions 
are converted in low-level representation by  the CalendarConverter web
service to be efficiently processed. Clearly, caching techniques can be used to optimize the process.

\section{Related Work}
\label{sec:rel}
Several formalisms have been proposed for symbolic
representation of granularities and periodicity.
Periodicity and its application in the AI and DB area have been
extensively investigated \cite{TC:IS03,MSK:CI96,KSW:PODS90,Ladkin:AAAI86}.
Regarding symbolic representation, it is well known
the formalism proposed by Leban et al. \citeyear{Leban-et-al:86}, that is based on the
notion of {\em collection}, and it is intended to represent temporal
expressions occurring in natural language.  A \emph{collection} is a
structured set of time intervals where the order of the collection
gives a measure of the structure depth: an order 1 collection is an
ordered list of intervals, and an order $n$ ($n>1$) collection is an
ordered list of collections having order $n-1$. Two operators, called
{\em slicing} and {\em dicing} are used to operate on collections by
selecting specific intervals or sub-collections, and by further
dividing an interval into a collection, respectively.  For example,
\texttt{Weeks:during:January2006} divides the interval corresponding
to \texttt{January2006} into the intervals corresponding to the weeks
that are fully contained in that month.  This formalism has been
adopted with some extensions by many researchers in the AI
\cite{Koomen91,cukierman:98} and Database area
\cite{Chandra-et-al:94,Ter:TKDE03}.  In particular,
the control statements \texttt{if-then-else} and
\texttt{while} have been introduced by Chandra et al.
\citeyear{Chandra-et-al:94} to facilitate the representation of
certain sets of intervals. For example, it is possible to specify:
\emph{the fourth Saturday of April if not an holiday, and the previous
  business day otherwise}.
%



{}As for the deductive database community, a second influential proposal 
is the \emph{slice} formalism introduced by Niezette et al. \citeyear{Niezette:92}. 
A slice denotes a (finite or infinite) set of
not necessarily consecutive time intervals.
For example, the slice
  \texttt{all.Years + \{2,4\}.Months +
    \{1\}.Days}~$\triangleright$~\texttt{2.Days} denotes a set of
  intervals corresponding to the first 2 days of February and April of
  each year.

A totally different approach is the {\em calendar algebra} described by Ning et al.
\citeyear{NWJ:amai02}, and considered in this paper. The representation is
based on a rich set of algebraic operators on periodic sets
as opposed to \textit{slicing} and \textit{dicing} over nonconvex intervals.

None of the above cited papers provide a mapping to identify how
each operator changes the mathematical characterization of the
periodicity of the argument expressions.  The problem of finding these
mappings is not trivial for some operators.

In \cite{BD:amai00} the expressive power of the algebras proposed by Leban et al.
\citeyear{Leban-et-al:86} and Niezette et al. \citeyear{Niezette:92} is compared and an
extension to the first is proposed in order to capture a larger set of
granularities. Since the periodical representation is used to compare
expressiveness, a mapping from calendar expressions in those
formalisms to periodical
representations can be found 
in the proofs of that
paper. However, since minimality is not an issue for the purpose of
comparing expressiveness, in many cases
the mapping returns non-minimal representations. 

Regarding alternative approaches for low-level representation, we
already mentioned that the ones based on strings
\cite{Wijsen:AAAI2000} and automata \cite{LagoMP03,BresolinMP04} may
be considered as an alternative for the target of our conversion.
As a matter of fact, an example of the conversion of a Calendar
Algebra expression into a string based representation can be found in
\cite{LagoM01}. A complete conversion procedure appeared during the
revision process of this paper in the PhD Dissertation by Puppis
\citeyear{PuppisThesis}.  The aim of the conversion is to prove that the
\emph{granspecs} formalism, used to represent granularities in terms
of automata, has at least the same expressiveness as the Calendar
Algebra.  Hence, obtaining minimal representations was not the goal.
Moreover, in their case minimization is not in terms of the period
length, but in terms of the automaton size and automaton complexity.
%
About the complexity of reasoning, given an automaton $M$, the worst
case time complexity of the operations analogous to our \emph{up}
and \emph{down} depends linearly on $||M||$, a value computed from $M$
itself and called \textit{complexity of $M$}.  In this sense $||M||$
has the same role of our period length ($P$), even if a precise
relationship between the two values is hard to obtain.  In our
approach we compute  \emph{up} in logarithmic time with respect to $P$ and
\emph{down} in linear time with respect to the dimension of the result
(that is bounded by $P$).  Other operations, like checking for
equivalence, seem to be more complex using automata
\cite{BresolinMP04}.
Techniques for minimization in terms of automaton complexity are
presented by Dal Lago et al. \citeyear{LagoMP03}, and the time complexity is proved to be
polynomial, even if the exact bound is not explicitly given.  In our
approach, the worst case time complexity for the minimization is
$O(P^{\frac{3}{2}})$ \cite{time05}.
Overall, the automata approach is very elegant and well-founded, but,
on one side it still misses an implementation in order to have some
experimental data to compare with, and on the other side only basic
operations have been currently defined; it would be interesting to
investigate the definition on that formalism of more complex
operations like the ones required by GSTP.


\section{Conclusion and Future Work}
\label{sec:conc}

We have presented an hybrid algorithm that
interleaves the conversion of Calendar Algebra subexpressions into periodical
sets with the minimization of the period length. We have proved that the algorithm returns
set-based granularity representations having minimal period length, which is extremely important for the efficiency of operations on granularities. 
Based on the technical contribution of this paper, a software system
is being developed allowing users to access multi-granularity reasoning
services by defining arbitrary time granularities with a high-level
formalism. 
Our current efforts are mainly devoted to completing and refining the
development of the different modules of the architecture shown in
Section~\ref{subsec:arch}.

As a future work, we intend to develop effective graphical user interfaces to
support the definition of Calendar Algebra expressions in a user friendly way.
Example~\ref{ex:ui} described one of the possible interfaces.
Another open issue is how to convert a periodical representation of a
granularity into a ``user friendly'' Calendar Algebra expression.
This conversion could be useful, for example, to present the
result of a computation performed using the periodical representation.
However, a naive conversion may not be effective since the resulting
calendar algebra expression could be as involved as the periodical
representation from which it is derived.
 For example, a conversion
procedure is presented by Bettini et al. \citeyear{BJW:book} to prove that the Calendar
Algebra is at least as expressive as the periodical representation;
however, the resulting Calendar Algebra expression is composed by a
number of Calendar Algebra operations that is linear in the number of
granules that are in one period of the original granularity.
On the contrary, an effective conversion should generate Calendar Algebra
expressions that are compact and easily readable by the user. This problem is somehow related
to the discovery of calendar-based association rules \cite{Ning:time01}.
Finally, we intend to investigate the usage of the automaton-based
representation as a low-level granularity formalism. It would be interesting to
know whether, using this representation, it is possible to
compute the same operations that can be computed with the periodical
representation and if any performance gain could be achieved.

\section*{Acknowledgments}
We thank the anonymous referees for their useful comments and suggestions.
The work of Bettini and Mascetti was partially supported by Italian MIUR
InterLink project N.II04C0EC1D. The work of Wang was partially supported by
the US NSF grant IIS-0415023.

\newpage
\appendix
\section{Proofs}
\label{ch:proofs}

\subsection{Transitivity of the \emph{Periodically Groups Into} Relationship}
In order to prove the correctness of the conversions of algebraic
expressions into periodical sets, it is useful to have a formal result
about the transitivity of the \emph{periodically groups into}
relation. In addition to transitivity of $\gpinto$, Theorem~\ref{theo:trans}
also says something about period length values.
%

\begin{theo}
\label{theo:trans}
Let $G$ and $H$ be two unbounded granularities such that $G$ is periodic in
terms of the bottom granularity (i.e., $\bot \gpinto G$) and $H$ is
periodic in terms of $G$ (i.e., $G \gpinto H$). Let $P_H^G$ and $N_H^G$ be the
period length and the period label distance of $H$ in terms of granules of $G$, and $N_G$ the period label distance of $G$ in terms of $\bot$.
Then, if $P_H^G=\alpha N_G$ for some positive integer
$\alpha$, then $H$ is periodic in terms of the bottom granularity
(i.e., $\bot \gpinto H$) and $P_H=\alpha P_G$.
\end{theo}

\begin{proof}
   Since by hypothesis $G \gpinto H$ and $P_H^G=\alpha N_G$, $\forall
  i$ if $H(i)=\bigcup_{r=0}^{n_i} G(i_r)$, then $H(i+N_H^G) = \bigcup_{r=0}^{n_i}G(i_r +
  \alpha N_G)$. This can be also written as follows:
  \\
  if
\begin{eqnarray}\label{HunionG}
        H(i)=G(i_0) \cup ... \cup G(i_{n_i})
\end{eqnarray}
then $\exists \beta \in \mathbb{N} s.t.$:
\begin{eqnarray}\label{H+bunionG}
        H(i+N_H^G)=G(i_0 + \alpha N_G) \cup ... \cup G(i_{n_i} + \alpha N_G)
\end{eqnarray}

Since $\bot \gpinto G$, if
\begin{eqnarray}\label{GunionB}
G(i_{j})=\bigcup_{k=0}^{\tau_{i_{j}}} \bot(i_{j,k})
\end{eqnarray}
then

\begin{eqnarray}\label{G+NunionG}
G(i_{j} + N_G) = \bigcup_{k=0}^{\tau_{i_{j}}} \bot(i_{j,k}+ P_G)
\end{eqnarray}

This can be clearly extended using $\alpha N_G$ instead of $N_G$.
\begin{eqnarray}\label{G+aNunionG}
G(i_{j} + \alpha N_G) = \bigcup_{k=0}^{\tau_{i_{j}}} \bot(i_{j,k}+ \alpha P_G)
\end{eqnarray}

Rewriting (\ref{HunionG}) substituting $G(i_j)$ according to (\ref{GunionB}) and rewriting (\ref{H+bunionG}) substituting $G(i_j+\alpha N_G)$ according to (\ref{G+aNunionG}), we obtain:\\
if
\( 
H(i)=\underbrace{\bot(i_{0,0}) \cup \ldots \cup \bot(i_{0,\tau_{i_{0}}})}_{G(i_{0})} \cup 
\ldots \cup \underbrace{\bot(i_{n_{i}, 0}) \cup ... \cup \bot(i_{n_{i}, \tau_{i_{n_{i}}}})}_{G(i_{n_{i}})}
\)
\\
then
\(
H(i+N_H^G)=\underbrace{\bot(i_{0,0}+\alpha P_G) \cup \ldots \cup \bot(i_{0,\tau_{i_{0}}}+\alpha P_G)}_{G(i_{0} + \alpha N_G)} \cup \ldots 
\\
\quad\quad \cup \underbrace{ \bot(i_{n_{i}, 0}+\alpha P_G) \cup \ldots \cup \bot(i_{n_{i}, \tau_{i_{n_{i}}}}+\alpha P_G)}_{G(i_{n_{i}}+\alpha N_G)}
\)
\\
Hence the second condition of Definition~\ref{def:pgroup} is satisfied. 
The third one is always satisfied for unbounded granularities.
The first one is satisfied too; in fact since $G \gpinto H$ with a period label distance of $N_H^G$, then for each label $i$ of $H$,  $i + N_H^G$ is a label of $H$.
Hence, by definition of periodically-groups-into $\bot \gpinto H$ with $P_H = \alpha P_G$ and $N_H = N_H^G$.
\end{proof}

\subsection{Proof of Proposition~\ref{prop:GroupOperation}}
\subsubsection{Part 1}
{}From the definition of the \textsl{Group} operation, for all $i \in \mathbb{N}$:
\[
G'(i)=\bigcup_{j=(i-1)m+1}^{im} G(j) = G(im-m+1) \cup \ldots \cup G(im)=G(\lambda) \cup \ldots \cup G(\lambda + m -1)
\]
with $\lambda = im - m + 1$. Furthermore, $\forall k \in \mathbb{N}$:
\[
G'(i+k)=\bigcup_{j=(i+k-1)m+1}^{(i+k)m} G(j) = G(im+km-m+1) \cup \ldots \cup G(im+km)=
\]
\[
= G(\lambda+km) \cup \ldots \cup G(\lambda + km+ m -1)
\]

\noindent
Hence,
\begin{eqnarray}\label{group}
\mbox{If } \;G'(i')=\bigcup_{r=0}^{m-1} G(\lambda + r)\; \mbox{ then } 
\; G'(i' + k) = \bigcup_{r=0}^{m-1} G(\lambda + r + km).
\end{eqnarray}
This holds for each $k$. If we use $k=\frac{N_G}{GCM(m, N_G)}$ (note that $k \in \mathbb{N}$), then all the hypotheses of Theorem~\ref{theo:trans} are satisfied:
(i) $\bot \gpinto G$ (by hypothesis);
(ii) $G \gpinto G'$ (since $G \groups G'$, $\Lset{G'}=\mathbb{Z}$, and (\ref{group}) holds);
(iii) $P_{G'}^G = \frac{m \cdot N_G}{GCM(m, N_G)}$ (since we use $k=\frac{N_G}{GCM(m, N_G)}$ and, from (\ref{group}) we know that $P_{G'}^G = km$).
Therefore, by Theorem~\ref{theo:trans}, $\bot \gpinto G'$ with $P_{G'} = \frac{mP_G}{GCM(m, N_G)}$ and $N_{G'} = \frac{N_G}{GCM(m \cdot N_G)}$.

\subsubsection{Part 2}
By definition of $l$, we need to show that $G'\left(\left\lfloor \frac{l_G - 1}{m} \right\rfloor + 1\right) = \bigcup_{j=b}^t G(j)$ with $b \leq l_G \leq t$.

{}From the definition of the \textsl{Group} operation, $G'(i)=\bigcup_{j=(i-1)\cdot m+1}^{i \cdot m} G(i)$ ; hence:
\[
G'\left(\left\lfloor \frac{l_{G}-1}{m}\right\rfloor + 1 \right)  = \bigcup_{j=\left\lfloor \frac{l_{G}-1}{m}\right\rfloor \cdot m +1} ^ {\left( \left\lfloor \frac{l_{G}-1}{m}\right\rfloor + 1\right) \cdot m} G(j)
\]

We prove the thesis showing that (1) $\left\lfloor \frac{l_{G}-1}{m}\right\rfloor \cdot m +1 \leq l_{G}$ and that (2) $\left( \left\lfloor \frac{l_{G}-1}{m}\right\rfloor + 1\right) \cdot m \geq l_{G}$.

(1) Since $\left\lfloor \frac{l_G-1}{m} \right\rfloor \leq \frac{l_G-1}{m}$, hence $\left\lfloor \frac{l_G-1}{m} \right\rfloor \cdot m + 1 \leq l_G$

(2) First we prove that $\left\lfloor \frac{l_G-1}{m}\right\rfloor \geq \frac{l_G}{m} - 1$. Since $\left\lfloor \frac{l_G-1}{m}\right\rfloor = \frac{l_G-1-\left[(l_G-1) mod \; m\right]}{m}$ we have to prove that $\frac{l_G-1-\left[(l_G-1) mod \; m\right]}{m} \geq \frac{l_G}{m} - 1$; it is equivalent to the inequality $-\left[(l_G-1) mod \; m\right] \geq - m + 1$ that is true since $(l_G-1) mod \; m \leq m -1$. Since $\left\lfloor \frac{l_G-1}{m}\right\rfloor \geq \frac{l_G}{m} - 1$ it is trivial that $\left( \left\lfloor \frac{l_{G}-1}{m}\right\rfloor + 1\right) \cdot m \geq l_{G}$.

\subsection{Proof of Proposition~\ref{pr:alterOperation}}
\subsubsection{Part 1}
\textbf{Proof sketch}\\
We show that $G_2 \gpinto G'$ with $P_{G'}^{G_2}= \alpha N_{G_2}$ and
then we apply Theorem~\ref{theo:trans} to obtain the thesis. In
particular we use
\[\Delta = lcm \left( N_{G_1}, m, \frac{P_{G_2} \cdot N_{G_1}}{GCD(P_{G_2} \cdot N_{G_1}, P_{G_1})}, \frac{N_{G_2} \cdot m}{GCD(N_{G_2} \cdot m, |k|)} \right)\]
\noindent and
\[\alpha = \left( \frac{\Delta \cdot P_{G_1} \cdot N_{G_2}}{N_{G_1} \cdot P_{G_2}} + \frac{\Delta \cdot k}{m}\right) \cdot \frac{P_{G_2}}{N_{G_2}}\]
such that, for each $i$, if $\exists j,k : G'(i) = \bigcup_{r=0}^k G_2(j + r)$,
then $G'(i+\Delta) = \bigcup_{r=0}^k G_2(j+r+\alpha N_{G_2})$.

Given an arbitrary granule $G'(i)$, we show that $G'(i+\Delta)$ is the
union of granules that can be obtained by adding $\alpha N_{G_2}$ to the
index of each granule of $G_2$ contained in $G'(i)$. 
Note that $i+\Delta\in\Lset{G'}$ since $G'$ is full-integer labeled.
 In order to show
that this is correct we consider the way granules of $G'$ are
constructed by definition of altering-tick.  More precisely, we
compute the difference between the label $b'_{i+\Delta}$ of the first granule of $G_2$
included in $G'(i+\Delta)$ and the label $b'_i$ of the first granule of $G_2$
included in $G'(i)$; we show that this difference is equal to the
difference between the label $t'_{i+\Delta}$ of the last granule of $G_2$ included in
$G'(i+\Delta)$ and the label $t'_i$ of the last granule of $G_2$ included in
$G'(i)$. This fact together with the consideration that $G_2$ is a
full-integer labeled granularity, leads to the conclusion that $G'(i)$
and $G'(i+\Delta)$ have the same number of granules. It is then clear that 
 the above computed label differences are also equal to the difference 
between the label of an arbitrary n-th granule of $G_2$ included in
$G'(i+\Delta)$ and the label of the n-th granule of $G_2$ included in
$G'(i)$. If this difference is $b'_{i+\Delta} - b'_{i}$, then we have:
if $\exists j,k : G'(i) = \bigcup_{r=0}^k G_2(j + r)$, then 
$G'(i+\Delta) = \bigcup_{r=0}^k G_2\left(j+r+ \left(b'_{i+\Delta} - b'_{i}\right)\right)$.
By showing that $b'_{i+\Delta} - b'_i$ is a multiple of $N_{G_2}$ the thesis follows. 

\textbf{Proof details}

Assume $G_1(i)=\bigcup_{j=b_i}^{t_i} G_2(j)$ and
$G_1(i+\Delta)=\bigcup_{j=b_{i+\Delta}}^{t_{i+\Delta}} G_2(j)$.  
We need to compute $b'_{i+\Delta} - b'_{i}$.  {}From the definition of the the altering-tick operation:
\begin{eqnarray}
b'_{i}=\left \{
                        \begin{array}{lll}
        b_{i}+\left(\left\lfloor \frac{i-l}{m}\right\rfloor\right)k & \mbox{if } i=\left(\left\lfloor            \frac{i-l}{m}\right\rfloor\right)m+l,\\
        \\
              b_{i}+\left(\left\lfloor \frac{i-l}{m}\right\rfloor+1\right)k     &\mbox{otherwise}.
      \end{array}
      \right.
\end{eqnarray}
and
\begin{eqnarray}
b'_{i+\Delta}=\left \{
                        \begin{array}{lll}
        b_{i+\Delta}+\left(\left\lfloor \frac{i+\Delta-l}{m}\right\rfloor\right)k & 
       \mbox{if } i+\Delta=\left(\left\lfloor \frac{i+\Delta-l}{m}\right\rfloor\right)m+l,\\
        \\
              b_{i+\Delta}+\left(\left\lfloor \frac{i+\Delta-l}{m}\right\rfloor+1\right)k     &\mbox{otherwise}.
      \end{array}
      \right.
\end{eqnarray}

Note that if $i=\left(\left\lfloor \frac{i-l}{m}\right\rfloor\right)m+l$,
then $i+\Delta=\left(\left\lfloor \frac{i+\Delta-l}{m}\right\rfloor\right)m+l$.
Indeed, $\left(\left\lfloor \frac{i+\Delta-l}{m}\right\rfloor\right)m+l = 
         \left(\left\lfloor \frac{i-l}{m} + \frac{\Delta}{m}\right\rfloor\right)m+l$ and,
since $\Delta$ is a multiple of $m$, then $\left(\left\lfloor \frac{i-l}{m} + 
              \frac{\Delta}{m}\right\rfloor\right)m+l = \left( \frac{\Delta}{m} + 
             \left\lfloor \frac{i-l}{m} \right\rfloor\right)m+l = 
              \Delta + \left(\left\lfloor \frac{i-l}{m}\right\rfloor\right)m+l$.

Hence, to compute $b'_{i+\Delta} - b'_{i}$ we should consider two cases:
\begin{eqnarray}
b'_{i+\Delta} - b'_{i}=\left \{
                        \begin{array}{lll}
        b_{i+\Delta}+\left(\left\lfloor \frac{i+\Delta-l}{m}\right\rfloor\right)k - b_{i}-\left(\left\lfloor \frac{i-l}{m}\right\rfloor\right)k 
        \;\; \mbox{if } i=\left(\left\lfloor \frac{i-l}{m}\right\rfloor\right)m+l\\
        \\
              b_{i+\Delta}+\left(\left\lfloor \frac{i+\Delta-l}{m}\right\rfloor+1\right)k - b_{i}-\left(\left\lfloor \frac{i-l}{m}\right\rfloor+1\right)k \;\; \mbox{otherwise}.
      \end{array}
      \right.
\end{eqnarray}

In both cases (again considering the fact that $\Delta$ is a multiple of $m$):
\begin{eqnarray}
\label{16}
b'_{i+\Delta} - b'_{i} = (b_{i+\Delta} - b_{i}) + \frac{\Delta \cdot k}{m}
\end{eqnarray}
We are left to compute $b_{i+\Delta} - b_{i}$, i.e., the distance in
terms of granules of $G_2$, between $G_2(b_{i})$ and
$G_2(b_{i+\Delta})$. Since, by hypothesis, $G_1(i) =
\bigcup_{j=b_{i}}^{t_{i}} G_2(j)$ and $G_1(i+\Delta) =
\bigcup_{j=b_{i+\Delta}}^{t_{i+\Delta}} G_2(j)$, then the first
granule of $\bot$ making $G_2(b_i)$ and the first granule of $\bot$
making $G_1(i)$ is the same granule. The same can be observed for the
first granule of $\bot$ making $G_2(b_{i+\Delta})$ and the first
granule of $\bot$ making $G_1(i+\Delta)$. More formally:
\[
min \left\lfloor b_i \right\rfloor ^{G_2} = min \left\lfloor i \right\rfloor ^{G_1}
\]
and
\[
min \left\lfloor b_{i+\Delta} \right\rfloor ^{G_2} = min \left\lfloor i+\Delta \right\rfloor ^{G_1}
\]
Hence, we have:
\begin{eqnarray}
min \left\lfloor b_{i+\Delta} \right\rfloor ^{G_2} - min \left\lfloor b_{i} \right\rfloor ^{G_2} = 
min \left\lfloor i+\Delta \right\rfloor ^{G_1} - min \left\lfloor i \right\rfloor ^{G_1}
\end{eqnarray}

We have shown that the difference between the index of the first
granule of $\bot$ making $G_2(b_{i+\Delta})$ and the index of the
first granule of $\bot$ making $G_2(b_i)$ is equal to the difference
between the index of the first granule of $\bot$ making
$G_1(i+\Delta)$ and the index of the first granule of $\bot$ making
$G_1(i)$.  Then, we need to compute the difference between the
index of the first granule of $\bot$ making $G_1(i+\Delta)$ and the
index of the first granule of $\bot$ making $G_1(i)$.  Since $\bot
\gpinto G_1$ and $\Delta$ is a multiple of $N_{G_1}$,
for each $i$, if $\exists j, \tau : G_1(i) =
\bigcup_{r=0}^{\tau} \bot(j+r)$, then $G_1(i+\Delta) = \bigcup_{r=0}^{\tau} \bot(j+\frac{\Delta \cdot P_{G_1}}{N_{G_1}})$. Hence, this difference has value 
$\frac{\Delta \cdot P_{G_1}}{N_{G_1}}$, and for what shown above this is also the value of
the difference between the index of the first
granule of $\bot$ making $G_2(b_{i+\Delta})$ and the index of the
first granule of $\bot$ making $G_2(b_i)$.
Then, since $\bot \gpinto G_2$ with period length $P_{G_2}$ and since $\frac{\Delta \cdot P_{G_1}}{N_{G_1}}$ is a multiple of $P_{G_2}$, we have that, if:
\[
\bot (j) \subseteq G_2(i)
\]
then:
\[
\bot(j+\frac{\Delta \cdot P_{G_1}}{N_{G_1}}) \subseteq G_2(i+\frac{\Delta \cdot P_{G_1} \cdot N_{G_2}}{N_{G_1} \cdot P_{G_2}})
\]

Thus, $b_{i+\Delta} - b_{i} = \frac{\Delta \cdot P_{G_1} \cdot N_{G_2}}{N_{G_1} \cdot P_{G_2}}$.

Reconsidering \ref{16}:
\[
b'_{i+\Delta} - b'_{i} = \frac{\Delta \cdot P_{G_1} \cdot N_{G_2}}{N_{G_1} \cdot P_{G_2}} + \frac{\Delta \cdot k}{m}.
\]

Analogously we can compute 
$t'_{i+\Delta} - t'_{i} = \frac{\Delta \cdot P_{G_1} \cdot N_{G_2}}{N_{G_1} \cdot P_{G_2}} + \frac{\Delta \cdot k}{m}$.

Thus, $b'_{i+\Delta} - b'_{i} = t'_{i+\Delta} - t'_{i}$; hence
$t_{i+\Delta} - b_{i+\Delta} = t_i - b_i$. Since $G_2$ is a full
integer labeled granularity, then $G'(i)$ and $G'(i+\Delta)$ are
formed by the same number of granules.

Since we now know $G'(i+\Delta) = \bigcup_{j=b'_{i+\Delta}}^{t'_{i+\Delta}} G_2(j)
=\bigcup_{j=b'_{i}}^{t'_{i}} G_2(j+(b'_{i+\Delta} - b'_i))$
and $(b'_{i+\Delta} - b'_i)$ is a multiple of $N_{G_2}$, we 
have $G_2 \gpinto G'$,  $P_{G'}^{G_2}= \frac{\Delta \cdot P_{G_1} \cdot N_{G_2}}{N_{G_1} \cdot P_{G_2}}$ and $\bot \gpinto G_2$. Hence, all the
hypothesis of Theorem~\ref{theo:trans} hold, and its application leads
the thesis of this proposition.

\subsubsection{Part 2}
Since $G_2$ partitions $G'$ (see table 2.2 of \cite{BJW:book}), then
(1) $\up{l_{G_2}}^{G'}_{G_2}$ is always defined and
(2) $min(\{n \in \mathbb{N^{+}} | \exists i \in \Lset{G_2} s.t.\; \bot(n) \subseteq G_2(i)\}) = min(\{m \in \mathbb{N^{+}} | \exists j \in \Lset{G'} s.t.\; \bot(m) \subseteq G'(j)\})$.
Therefore $l_{G'}$ is the label of the granule of $G'$ that covers the granule of $G_2$ labeled with $l_{G_2}$; by definition of $\up{\cdot}$ operation, $l_{G'}=\up{l_{G_2}}^{G'}_{G_2}$.

\subsection{Proof of Proposition~\ref{pr:shiftOperation}}
\subsubsection{Part 2}
By definition of the \textsl{Shift} operation, $G'(i) = G(i-m)$. Hence $G'(l_G + m) = G(l_G + m - m) = G(l_G)$.

\subsection{Proof of Proposition~\ref{pr:combineOperation}}
\subsubsection{Part 1}
The thesis will follow from the application of
Theorem~\ref{theo:trans}. Indeed, we know that $\bot \gpinto G_2$ and we show that $G_2 \gpinto G'$ with
$P_{G'}^{G_2}$ multiple of $N_{G_2}$. For this we need to identify
$\Delta$ and $\alpha$ s.t., for each $i$, if there exists $s(i)$ s.t.
$G'(i) = \bigcup_{j\in s(i)} G_2(j)$, then $G'(i+\Delta) = \bigcup_{j
  \in s(i)} G_2(j+\alpha N_{G_2})$.

Consider an arbitrary $i\in\mathbb{N}$ and $\Delta= \frac{lcm(P_{G_1},
  P_{G_2})N_{G_1}}{P_{G_1}}$.  By definition of the combining
operation, we have $G'(i)=\bigcup_{j\in s(i)} G_2(j)$ and
$G'(i+\Delta)=\bigcup_{j\in s(i+\Delta)}G_2(j)$ with
\[s(i)=\left\{ j \in {\cal L}_{G_2} | \emptyset \neq G_2(j) \subseteq G_1(i)
\right\}\] and
\[s(i+\Delta)=\left\{ j \in {\cal L}_{G_2} | \emptyset
  \neq G_2(j) \subseteq G_1(i+\Delta) \right\}.\]

We now show that $s(i+\Delta)$ is composed by all and only the elements of $s(i)$ when the quantity $\Delta' = \frac{lcm (P_{G_1}, P_{G_2}) N_{G_2}}{P_{G_2}}$ is added. For this purpose we need:

\begin{eqnarray}\label{comb1}
\forall j\in s(i) \ \exists (j+\Delta') \in s(i+\Delta)
\end{eqnarray}
and
\begin{eqnarray}\label{comb2}
\forall \left(j + \Delta' \right) \in s(i+\Delta) \ \exists j\in s(i)
\end{eqnarray}

About \ref{comb1}, note that if $j\in s(i)$, then $G_2(j) \subseteq G_1(i)$. Since $\bot \gpinto G_2$, if
\[
G_2(j)=\bigcup_{r=0}^{k} \bot(j_{r})
\]
then
\begin{eqnarray}\label{comb3}
G_2\left(j+\Delta'\right) = \bigcup_{r=0}^{k} \bot(j_r+lcm (P_{G_1}, P_{G_2}))
\end{eqnarray}
Since $G_1(i) \supseteq G_2(j) = \bigcup_{r=0}^{k} \bot(j_{r})$, and since $\bot \gpinto G_1$, then
\begin{eqnarray}\label{comb4}
G_1(j+\Delta) \supseteq \bigcup_{r=0}^{k} \bot(j_r+lcm (P_{G_1}, P_{G_2}))
\end{eqnarray}

{}From \ref{comb3} and \ref{comb4} we derive $G_1(i+\Delta)\supseteq
G_2(j+\Delta')$, and hence $(j+\Delta') \in s(i+\Delta)$.  Analogously
can be proved the validity of \ref{comb2}; Hence, for each $i$, if
there exists $s(i)$ s.t. $G'(i) = \bigcup_{j\in s(i)} G_2(j)$, then
$G'(i+\Delta) = \bigcup_{j \in s(i)} G_2(j+\Delta')$. Hence,
considering the fact that $G_2 \groups G'$, we can conclude $G_2
\gpinto G'$.  Finally, since $P_{G'}^{G_2}$ is a multiple of
$N_{G_2}$, by Theorem~\ref{theo:trans} we obtain the thesis.

\subsubsection{Part 2}
Let
\[
\Ltil{G'} = \{ i \in \Lcap{G_1}^{P_{G'}} | \widetilde{s}(i) \neq \emptyset\}
\]
where
$\forall i \in \Lcap{G_1}^{P_{G'}} \; \widetilde{s}(i) = \{j \in \Lcap{G_2}^{P_{G'}} | \emptyset \neq G_2(j) \subseteq G_1(i) \}$;

We show that $\Ltil{G'} = \Lcap{G'}$ by proving that:
(1) $\Ltil{G'} \supseteq \Lcap{G'}$ and
(2) $\Ltil{G'} \subseteq \Lcap{G'}$.

(1) Suppose by contradiction that exists $k \in \Lcap{G'} \setminus \Ltil{G'}$. Since $k \in \Lcap{G'}$ and since $G'$ is derived by the \textsl{Combine} operation, then
$\exists q \in \Lset{G_2} | G_2(q) \subseteq G_1(k)$.
By definition of the \textsl{Combine} operation $G'(k) = \bigcup_{j \in s(k)} G_2(j)$; since $q \in s(k)$, then $G_2(q) \subseteq G'(k)$.
Hence (a) $\exists q \in \Lset{G_2} | G_2(q) \subseteq G'(k)$.

Moreover, since $k \not\in \Ltil{G'}$, then $\widetilde{s}(k) = \emptyset$; therefore $\nexists \textbf{j} \in \Lcap{G_2}^{P_{G'}} | G_2(j) \subseteq G_1(k)$.
By definition of the \textsl{Combine} operation it is easily seen that $G' \finer G_1$. Using this and the previous formula, we derive that (b) $\nexists j \in \Lcap{G_2}^{P_{G'}} | G_2(j) \subseteq G'(k)$.

{}From (a) and (b) it follows that $\exists q \in \Lset{G_2} \setminus \Lcap{G_2}^{P_{G'}} | G_2(q) \subseteq G'(k)$. We show that this leads to a contradiction.

Since $q \not\in \Lcap{G_2}^{P_{G'}}$ and labels of $\Lcap{G_2}^{P_{G'}}$ are contiguous (i.e., $\nexists i \in \Lset{G_2} \setminus \Lcap{G_2}^{P_{G'}}$ s.t. $min(\Lcap{G_2}^{P_{G'}}) < i < max(\Lcap{G_2}^{P_{G'}})$), then $q<min(\Lcap{G_2}^{P_{G'}})$ or $q>max(\Lcap{G_2}^{P_{G'}})$. We consider the first case, the proof for the second is analogous.

If $q<min(\Lcap{G_2}^{P_{G'}})$ then $max(\down{q}^{G_2})<1$ (otherwise $q \in \Lcap{G_2}^{P_{G'}}$).

Let be $\alpha = min (\down{min(\Lcap{G'})}^{G'})$. Since $k \in \Lcap{G'}$, then $\alpha \leq \down{k}^{G'}$.

If $\alpha \geq 1$, then $G'(k) \cap G_2(q) = \emptyset$ contradicting $G'(k) \supseteq G_2(q)$.

If $\alpha < 1$, then $G'(l_{G'}) \supseteq \bot(0)$ and we show that $l_{G'} \in \Ltil{G'}$.
Indeed, by definition of \textsl{Combine}, $\exists j \in \Lcap{G_2}^{P_{G'}} | G_2(j) \subseteq G'(L_{G'})$.
Since $G' \finer G_1$ we also have $\exists j \in \Lcap{G_2}^{P_{G'}} | G_2(j) \subseteq G_1(L_{G'})$; hence $j \in \widetilde{s}(l_{G'})$ and then $l_{G'} \in \Ltil{G'}$.

Since $0 \in G'(l_{G'})$ and $max(\down{q}^{G_2}) \leq 0$, then $max(\down{q}^{G_2}) < \alpha$ (otherwise $G_2(q) \subseteq G'(l_{G'})$). Therefore, since $min(\down{k}^{G'}) \geq \alpha$, then $\down{q}^{G_2} \cap \down{l_{G'}}^{G'} = \emptyset$, in contradiction with $G_2(q) \subseteq G'(k)$.

(2) Suppose by contradiction that $\exists k \in \Ltil{G'} \setminus \Lcap{G'}$.
Since $k \in \Ltil{G'}$, by definition of $\Ltil{}$, $k \in \Lcap{G_1}^{P_{G'}}$ and $\widetilde{s}(k) \neq \emptyset$; Therefore, by definition of $\widetilde{s}$, $\exists j \in \Lcap{G_2}^{P_{G'}} | G_2(j) \subseteq G_1(k)$.

Since $j \in \Lcap{G_2}^{P_{G'}}$, by definition of $\Lcap{}$, $\exists h$ with $0<h\leq P_{G'}$ s.t. $\up{h}^{G_2} = j$. Since $G_2(j) \subseteq G_1(k)$, then $\up{h}^{G_1} = k$. By definition of the \textsl{combine} operation, $\up{h}^{G'}=k$. Moreover, since $0<h \leq P_{G'}$, by definition of $\Lcap{}$, $\up{h}^{G'} = k \in \Lcap{G'}$, contradicting the hypothesis.
 
\subsection{Proof of Proposition~\ref{pr:anchorOperation}}

\subsubsection{Part 1}
The thesis will follow from the application of
Theorem~\ref{theo:trans}. Indeed, we show that
$G_1 \gpinto G'$ with $P_{G'}^{G_1}$ multiple of $N_{G_1}$. 
For this we need to identify
$\Delta$ and $\alpha$ s.t., for each $i$, if there exists $s(i)$ s.t.
$G'(i) = \bigcup_{j\in s(i)} G_1(j)$, then $G'(i+\Delta) = \bigcup_{j
  \in s(i)} G_1(j+\alpha N_{G_1})$.
Let $\Delta = \frac{lcm(P_{G_1}, P_{G_2})N_{G_2}}{P_{G_2}}$. 
By definition of anchored
grouping, $G'(i) = \bigcup_{j=i}^{i'-1} G_1(j)$ and
$G'(i+\Delta)=\bigcup _{j=i+\Delta}^{(i+\Delta)'-1}G_1(j)$ where $i'$
is the first label of $G_2$ after $i$ and $(i+\Delta)'$ is the first label of 
$G_2$ after $i+\Delta$.
By periodicity of $G_2$, (and since $\Delta$ is a multiple of $N_{G_2}$)
the difference between the label of the granule following $G_2(i +
\Delta)$ and the label of the granule following $G_2(i)$ is $\Delta$. More formally,
$(i+\Delta)' - i' = \Delta$, hence $(i+\Delta)'=i'+\Delta$.
Then, for each $i$, if $G'(i) = \bigcup_{j=i}^k G_1(j)$, then
$G'(i+\Delta)=\bigcup_{j=i+\Delta}^{i'+\Delta-1}G_1(j)=\bigcup_{j=i}^{i'-1}G_1(j+\Delta)$.
By this result and considering $G_1 \groups G'$, we conclude $G_1\gpinto G'$ with $P_{G'}^{G_1}=\Delta$.
Note that by Proposition~\ref{prop:SetOperations1}, $N_{G_1}=\frac{P_{G_1}
  \cdot N_{G_2}}{P_{G_2}}$, hence $P_{G'}^{G_1}$ is a multiple of
$\Delta$. Then, by Theorem~\ref{theo:trans}, we have the thesis.

\subsubsection{Part 2}
Let
\[\Ltil{G'}=\left \{ \begin{array}{ll}
      \Lcap{G_2}^{P_{G'}}, & \mbox{if } l_{G_2}=l_{G_1},\\
      \ \{l'_{G_2}\} \cup \Lcap{G_2}^{P_{G'}},     &\mbox{otherwise},
      \end{array}
      \right.
\]
        where $l'_{G_2}$ is the greatest among the labels of $\Lset{G_2}$ that are smaller than $l_{G_2}$.
        We show that $\Ltil{G'} = \Lcap{G'}$ by proving that (1) $\Ltil{G'} \subseteq \Lcap{G'}$ and (2) $\Lcap{G'} \subseteq \Ltil{G'}$.
        
        (1) Suppose by contradiction that $\exists k \in \Ltil{G'} \setminus \Lcap{G'}$. Then, since $k \in \Ltil{G'}$, then $k \in \Lcap{G_2}^{P_{G'}}$ or $k=l_{G_2}'$.
        
        If $k \in \Lcap{G_2}^{P_{G'}}$, then, by definition of $\Lcap{G_2}^{P_{G'}}$, $\exists h$ with $0 < h \leq P_{G'}$ s.t. $\up{h}^{G_2} = k$. By definition of \textsl{Anchored-group}, $G'(k)=\bigcup_{j=k}^{k'-1}G_1(j)$ where $k'$ is the first label of $G_2$ after $k$. Therefore $G'(k) \supseteq G_1(k)$. Since $G_2$ is a labeled aligned subgranularity of $G_1$ and since $k \in \Lset{G_2}$, then $k \in \Lset{G_1}$ and $G_1(k)=G_2(k)$. Hence $G'(k) \supseteq G_2(k)$. It follows that $\up{h}^{G'}=k$ and therefore, by definition of $\Lcap{}$, $k \in \Lcap{G'}$ in contrast with the hypothesis.

        If $k = l_{G_2}'$, then, by definition of $\Ltil{G'}$, $l_{G_2} \neq l_{G_1}$. Therefore, since $G_2$ is a labeled aligned subgranularity of $G_1$ $l_{G_2}' < l_{G_1} < l_{G_2}$; then $\exists h$ with $0<h < min(\down{l_{G_2}}^{G_2})$ s.t. $\up{h}^{G_1}=l_{G_1}$. Since, by definition of \textsl{Anchored-group}, $G'(l_{G_2}')=\bigcup_{j=l_{G_2}'}^{l_{G_2}-1}G_1(j)$ and since $l_{G_2}' < l_{G_1} < l_{G_2}$, then $G'(l_{G_2}') \supseteq G_1(l_{G_1})$. Hence $\up{h}^{G'} = l_{G_2}'$ and therefore, by definition of $\Lcap{}$, $l_{G_2}'=k \in \Lcap{G'}$ in contrast with the hypothesis.
        
        (2) Suppose by contradiction that $\exists k \in \Lcap{G'}
        \setminus \Ltil{G'}$.  If $k \in \Lcap{G_2}^{P_{G'}}$ then, by
        definition of $\Ltil{G'}$, $k \in \Ltil{G'}$, in contrast with
        the hypothesis.
        
        If $k \notin \Lcap{G_2}^{P_{G'}}$, since $\nexists q \in
        \Lset{G_2} \setminus \Lcap{G_2}^{P_{G'}}$
        s.t. $min(\Lcap{G_2}^{P_{G'}}) \leq q \leq
        max(\Lcap{G_2}^{P_{G'}})$, then $k > max(\Lcap{G_2}^{P_{G'}})$
        or $k < min(\Lcap{G_2}^{P_{G'}})$.
        
        If $k > max(\Lcap{G_2}^{P_{G'}})$ then, by definition of
        $\Lcap{}$, $min(\down{k}^{G_2}) > P_{G'}$. Since $G_2$ is a
        labeled aligned subgranularity of $G_1$ then $G_2(k)=G_1(k)$
        and hence $min(\down{k}^{G_1})> P_{G'}$. Since $G'(k) =
        \bigcup_{j=k}^{k'-1} G_1(j)$ then $min(\down{k}^{G'})>P_{G'}$
        in contrast with the hypothesis $k \in \Lcap{G'}$.
        
        If $k < min(\Lcap{G_2}^{P_{G'}})$ then, by definition of
        $l_{G_2}'$, $k<l_{G_2}'$ or $k=l_{G_2}'$.
        
        If $k<l_{G_2}'$ then, let $k'$ be the next label of $G_2$
        after $k$. Since $k < l_{G_2}'$ then, by definition
        $l_{G_2}'$, $k' \leq l_{G_2}'$. By definition of $l_{G_2}'$
        then $max(\down{l_{G_2}'}^{G_2}) \leq 0$.  Since $G_2$ is a
        labeled aligned subgranularity of $G_1$ then
        $G_1(l_{G_2}')=G_2(l_{G_2}')$; therefore
        $max(\down{l_{G_2}'}^{G_1}) \leq 0$. Since $G'(k) =
        \bigcup_{j=k}^{k'-1}G_1(j)$ and $k' \leq l'_{G_2}$, follows
        that $max(\down{k}^{G'}) \leq 0$ in contrast with the
        hypothesis $k \in \Lcap{G'}$.
        
        Finally if $k=l_{G_2}'$ then $G'(l_{G_2}') =
        \bigcup_{j=l_{G_2}'}^{l_{G_2}-1} G_1(j)$. Since $k = l_{G_2}'
        \in \Lcap{G'}$ then $\exists h$ with $0 < h \leq P_{G'}$
        s.t. $\up{h}^{G'}=l_{G_2}'$. Since $G'$ is the composition of
        granules of $G_1$, $\up{h}^{G_1}$ is defined. Let $q =
        \up{h}^{G_1}$. By definition of $\Lcap{}$, $q \in
        \Lcap{G_1}^{P_{G'}}$ and therefore $q \geq l_{G_1}$.  Since,
        by definition of \textsl{Anchored-group}, $G'$ is the
        composition of granules of $G_1$ and since
        $\up{h}^{G'}=l_{G_2}'$ and $\up{h}^{G_1}=q$, then $G_1(q)
        \subseteq G'(l_{G_2}')$. Therefore since
        $G'(l_{G_2}')=\bigcup_{j=l_{G_2}'}^{l_{G_2}-1}G_1(j)$ then
        $q<l_{G_2}$.  It follows that $l_{G_1} \leq q < l_{G_2}$ and
        hence $l_{G_1} \neq l_{G_2}$. By definition of $\Ltil{G'}$,
        $l_{G_2}'=k \in \Ltil{G'}$ in contrast with the hypothesis.

\subsection{Selecting operations}
\label{sec:commonSelectingOperation}
The selecting operations have a common part in the proof for the computation of the period length and the period label distance.

Let be $\Gamma=\frac{lcm(P_{G_1}, P_{G_2})N_{G_1}}{P_{G_1}}$.  The
proof is divided into two steps: first we show that for each
select operation if $i \in \Lset{G'}$ then $i + \Gamma
\in \Lset{G'}$ (details for \textsl{Select-down},
\textsl{Select-up} and \textsl{Select-by-intersect} operations can be
found below). The second step is the application of
Theorem~\ref{theo:trans}. Indeed, for each \textsl{Select} operation,
the following holds: $\forall i \in \Lset{G'} \ G'(i)=G_1(i)$; this
implies $G_1 \groups G'$. {}From step 1 follows that $i + \Gamma \in
\Lset{G'}$, hence $G'(i+ \Gamma)=G_1(i+ \Gamma)$. By this result
and considering $G_1 \groups G'$, we conclude that $G_1 \gpinto G'$
with $P_{G'}^{G_1}=\Gamma$ which is a multiple of $N_{G_1}$ by
definition. Then, by Theorem~\ref{theo:trans} we have the thesis.
        
\subsection{Proof of Proposition~\ref{pr:selectDownOperation}}
\subsubsection{Part 1}
See Section~\ref{sec:commonSelectingOperation}.

We prove that if $\lambda \in {\cal L}_{G'}$ then $\lambda' = \lambda
+ \Gamma \in {\cal L}_{G'}$.

By definition of the \texttt{select-down} operation, if $\lambda \in
{\cal L}_{G'}$ then $\exists i \in {\cal L}_{G_2}$ s.t. $\lambda \in
\Delta_k^l\left(S(i)\right)$ where $S(i)$ is an ordered set defined as
follows: $S(i)=\{j \in {\cal L}_{G_1} | \emptyset \neq G_1(j)
\subseteq G_2(i)\}$. In order to prove the thesis we need to show that
$\exists i' \in {\cal L}_{G_2} | \lambda' \in \Delta_k^l(S(i'))$.
Consider $i'=i+\frac{lcm(P_{G_1} P_{G_2})N_{G_2}}{P_{G_2}}$ we will
note that $i' \in {\cal L}_{G_2}$ (this is trivially derived from the
periodicity of $G_2$). To prove that $\lambda' \in \Delta_k^l(S(i'))$
we show that $S(i')$ is obtained from $S(i)$ by adding $\Gamma$ to
each of its elements.

Indeed note that from periodicity of $G_1$, $\forall j \in S(i)$ if:
\begin{eqnarray}\label{selectdown3}
G_1(j)=\bigcup_{r=0}^{\tau_{j}} \bot(j_{r})
\end{eqnarray}

then:
\begin{eqnarray}\label{selectdown4}
G_1\left(j'\right)=\bigcup_{r=0}^{\tau_{j}} \bot(j_{r}+lcm(P_{G_1} P_{G_2}))
\end{eqnarray}

Since $j \in S(i)$, $G_1(j)\subseteq G_2(i)$ then, from (\ref{selectdown3}), $G_2(i) \supseteq \bigcup_{r=0}^{\tau_{j}} \bot(j_{r})$. Moreover, from periodicity of $G_2$:

\begin{eqnarray}\label{selectdown5}
G_2\left(i'\right) \supseteq \bigcup_{r=0}^{\tau_{j}} \bot(j_{r}+lcm(P_{G_1} P_{G_2}))
\end{eqnarray}

Since (\ref{selectdown4}) and (\ref{selectdown5}), $G_2\left(i'\right) \supseteq G_1\left(j'\right)$; hence $\forall j \in S(i) ,  j'=(j + \Gamma) \in S(i')$. Analogously we can prove that $\forall j' \in S(i') , j=(j' - \Gamma) \in S(i)$.

Thus $S(i')$ is obtained from $S(i)$ by adding $\Gamma$ to each of its elements; therefore if $j \in S(i)$ has position $n$ in $S(i)$, so $j' \in S(i')$ has position $n$ in $S(i')$. Hence it is trivial that if $\lambda$ has position between $k$ and $k+l-1$ in $S(i)$, then $\lambda'$ has position between $k$ and $k+l-1$ in $S(i')$. Hence if $\lambda \in {\cal L}_{G'}$, then $\lambda' \in {\cal L}_{G'}$.

\subsubsection{Part 2}
Let
\[
\Ltil{G'}=\bigcup_{i \in \Lcap{G_2}^{P_{G'}}} \left\{a \in A(i) | a \in \Lcap{G_1}^{P_{G'}}\right\};
\]
where $\forall i \in \Lset{G_2}$:
\[
A(i) =  \Delta_{k}^{l} \left( \left\{ j \in \Lset{G_1} | \emptyset \neq G_1(j) \subseteq G_2(i) \right\} \right).
\]

We show that $\Ltil{G'} = \Lcap{G'}$ by proving that 
(1) $\Ltil{G'} \subseteq \Lcap{G'}$ and
(2) $\Ltil{G'} \supseteq \Lcap{G'}$.

(1)Suppose by contradiction that $\exists q \in \Ltil{G'} \setminus \Lcap{G'}$. By definition of $\Ltil{G'}$, $q \in \Lcap{G_1}^{P_{G'}}$; therefore $\exists h$ with $0<h \leq P_{G'}$ s.t. $\up{h}^{G_1}=q$.
Moreover, by definition of $\Ltil{G'}$ and by definition of \textsl{Select-down}, $\Ltil{G'} \subseteq \Lset{G'}$ hence $q \in \Lset{G'}$.
Since, by definition of \textsl{Select-down} $G'(q) = G_1(q)$, then $\up{h}^{G'} = q$; hence, by definition of $\Lcap{}$, $q \in \Lcap{G'}$ in contradiction with hypothesis.

(2)Suppose by contradiction that $\exists q \in \Lcap{G'} \setminus \Ltil{G'}$. Since $q \in \Lcap{G'}$ then, by definition of \textsl{Select-down}
\[
\exists i \in \Lset{G_2} \, s.t. \, q \in \Delta_k^l\left(\left\{ j \in \Lset{G_1} | \emptyset \neq G_1(j) \subseteq G_2(i) \right\}\right)
\]
therefore, by definition of $A(i)$, $q \in A(i)$.

Since $q \in \Lcap{G'}$ then $\exists h$ with $0<h \leq P_{G'}$ s.t. $\up{h}^{G'} = q$. By definition of \textsl{Select-down}, $G'(q) = G_1(q)$, then $\up{h}^{G_1} = q$ and therefore $q \in \Lcap{G_1}^{P_{G'}}$. Moreover, since $G_1(q) \subseteq G_2(i)$, then $\up{h}^{G_2}=i$ and therefore $i \in \Lcap{G_2}^{P_{G'}}$.
Since $q \in A(i)$, $q \in \Lcap{G_1}^{P_{G'}}$ and $i \in \Lcap{G_2}^{P_{G'}}$ then, by definition of $\Ltil{G'}$, $q \in \Ltil{G'}$, in contrast with the hypothesis.

\subsection{Proof of Proposition~\ref{pr:selectUpOperation}}
\subsubsection{Part 1}
See Section~\ref{sec:commonSelectingOperation}.
We prove that if $i \in {\cal L}_{G'}$ then $i + \Gamma \in {\cal
  L}_{G'}$. {}From the periodicity of $G_1$, $i + \Gamma \in {\cal
  L}_{G_1} $ (this is trivially derived from the periodicity of
$G_1$). Hence we only need to show that $\exists j' \in {\cal L}_{G_2}
| \emptyset \neq G_2(j)\subseteq G_1(i + \Gamma)$. Since $i \in {\cal
  L}_{G'}$ then $\exists j \in {\cal L}_{G_2} | \emptyset \neq
G_2(j)\subseteq G_1(i)$.

{}From the periodicity of $G_2$, if: \begin{eqnarray}\label{selectup1} G_2(j)=\bigcup_{r=0}^{\tau_{j}} \bot(j_{r})
\end{eqnarray}

then: 
\begin{eqnarray}\label{selectup2}
G_2\left(j+\frac{lcm(P_{G_1} P_{G_2})N_{G_2}}{P_{G_2}}\right)=\bigcup_{r=0}^{\tau_{j}} \bot(j_{r}+lcm(P_{G_1} P_{G_2}))
\end{eqnarray}

Moreover, from the (\ref{selectup1}) and since $G_1(i) \supseteq G_2(j)$:
\[
G_1(i) \supseteq \bigcup_{r=0}^{\tau_{j}} \bot(j_{r})
\]

{}From the periodicity of $G_1$:

\begin{eqnarray}\label{selectup3}
G_1(i+\Gamma) \supseteq \bigcup_{r=0}^{\tau_{j}} \bot(j_{r}+lcm(P_{G_1} P_{G_2}))
\end{eqnarray}

{}From (\ref{selectup2}) and (\ref{selectup3}) follows that $G_1(i+\Gamma) \supseteq G_2\left(j+\frac{lcm(P_{G_1} P_{G_2})N_{G_2}}{P_{G_2}}\right)$, that is the thesis.

\subsubsection{Part 2}
Let 
\[
\Ltil{G'}=\{i \in \Lcap{G_1}^{P_{G'}} | \exists j \in \Lset{G_2} \, s.t. \, \emptyset \neq G_2(j) \subseteq G_1(i) \};
\]

We show that $\Ltil{G'} = \Lcap{G'}$ by proving that 
(1) $\Ltil{G'} \subseteq \Lcap{G'}$ and
(2) $\Ltil{G'} \supseteq \Lcap{G'}$.

(1) Suppose by contradiction that $\exists k \in \Ltil{G'} \setminus \Lcap{G_2}$.
Since $k \in \Ltil{G'}$, then $k \in \Lcap{G_1}^{P_{G'}}$; therefore $\exists h$ with $0<h \leq P_{G'}$ s. t. $\up{h}^{G_1} = k$.
Moreover, by definition of $\Ltil{G'}$ and by definition of \textsl{Select-down}, $\Ltil{G'} \subseteq \Lset{G'}$ hence $q \in \Lset{G'}$.
Since, by definition of \textsl{Select-up}, $G'(k) = G_1(k)$, then $\up{h}^{G'} = k$. Hence, by definition of $\Lcap{}$, $k \in \Lcap{G'}$, in contrast with the hypothesis.

(2) Suppose by contradiction that $\exists k \in \Lcap{G'} \setminus \Ltil{G'}$.
Since $k \in \Lcap{G'}$, then $\exists h$ with $0<h \leq P_{G'}$ s.t. $\up{h}^{G'} = k$.
Since, by definition of \textsl{Select-up}, $G'(k)=G_1(k)$, then $\up{h}^{G_1} = k$; Therefore, by definition of $\Lcap{}$, $k \in \Lcap{G_1}^{P_{G'}}$.
Moreover, since $k \in \Lcap{G'}$ and $\Lcap{G'} \subseteq \Lset{G'}$, by definition of the \textsl{Select-up} operation, then $\exists j \in \Lset{G_2}$ s.t. $\emptyset \neq G_2(j) \subseteq G_1(k)$.
Hence by definition of $\Ltil{G'}$, $k \in \Ltil{G'}$, in contradiction with hypothesis.

\subsection{Proof of Proposition~\ref{pr:selectByIntersectOperation}}
\subsubsection{Part 1}
See Section~\ref{sec:commonSelectingOperation}.
We prove that if $\lambda \in {\cal L}_{G'}$, then $\lambda'=\lambda + \Gamma \in {\cal L}_{G'}$.

By definition of the \textsl{select-by-intersect} operation, if $\lambda \in {\cal L}_{G'}$,
then $\exists i \in {\cal L}_{G_2} : \lambda \in \Delta_k^l(S(i))$
where $S(i)$ is an ordered set defined as follows:
$S(i)=\{j \in {\cal L}_{G_1} | G_1(j) \cap G_2(i) \neq \emptyset \}$.
In order to prove the thesis we need to show that $\exists i' \in {\cal L}_{G_2} : \lambda' \in \Delta_k^l(S(i'))$. Consider $i'=i+\frac{lcm(P_{G_1} P_{G_2})N_{G_2}}{P_{G_2}}$ note that $i' \in {\cal L}_{G_2}$ (this is trivially derived from the periodicity of $G_2$). To prove that $\lambda' \in \Delta_k^l(S(i'))$ we show that $S(i')$ is obtained from $S(i)$ by adding $\Gamma$ to each of its elements.

Indeed note that $\forall j$ if $j \in S(i)$, then $G_1(j) \cap G_2(i)
\neq \emptyset$. Hence $\exists l \in \mathbb{Z} : \bot(l)\subseteq
G_1(j)$ and $\bot(l)\subseteq G_2(i)$. {}From the periodicity of $G_1$,
$G_1(j+\Gamma) \supseteq \bot(l+lcm(P_{G_1} P_{G_2}))$. {}From the
periodicity of $G_2$, $G_2(i')\supseteq \bot(l+lcm(P_{G_1} P_{G_2}))$.
So $G_1(j+\Gamma) \cap G_2(i') \neq \emptyset$, therefore $\forall j
\in S(i) , (j + \Gamma) \in S(i')$.

Analogously we can prove that $\forall j' \in S(i') , (j' - \Gamma)
\in S(i)$. Hence $S(i')$ is obtained from $S(i)$ by adding $\Gamma$ to
each of its elements. Therefore, if $j \in S(i)$ has position $n$ in
$S(i)$, then $j+\Gamma \in S(i')$ has position $n$ in $S(i')$; hence
if $j$ has position between $k$ and $k+l-1$ in $S(i)$, then also
$j+\Gamma$ has position between $k$ and $k+l-1$ in $S(i')$ and so
$j+\Gamma \in {\cal L}_{G'}$.

\subsubsection{Part 2}
The proof is analogous to the ones of Proposition~\ref{pr:selectDownOperation}.

\subsection{Set Operations}

\subsubsection{Proof of Proposition~\ref{prop:SetOperations1}}
Given the periodical granularities H and G with G label aligned subgranularity of H, we prove that $\frac{N_G}{P_G}=\frac{N_H}{P_H}$. The thesis is proved by considering the common period length of $H$ and $G$ i.e. $P_c = lcm(P_G, P_H)$.

Let $N'_G$ be the difference between the label of the $i^{th}$ granule of one period of $G$ and the label of the $i^{th}$ granule of the next period, considering $P_c$ as the period length of $G$. Analogously $N'_{H}$ is defined. 

By periodicity of $G$, if $G(i) = \bigcup_{r=0}^{k}\bot(i_r)$ then $G(i+N'_G)=\bigcup_{r=0}^{k}\bot(i_r + P_c)$; since $G$ is an aligned subranularity of H, $\forall i \in {\cal L}_{H} \; H(i) = G(i) = \bigcup_{r=0}^{k}\bot(i_j)$ and, since $H$ is periodic, $H(i+N'_H)=\bigcup_{r=0}^{k}\bot(i_j + P_c)$; from which we can easily derive 
that $i+N'_G=i+N'_H$, hence $N'_G=N'_H$. 

{}From the definition of $P_c$, $\exists \alpha, \beta \in \mathbb{N}$
s. t. $\alpha P_H = \beta P_G$. Moreover, since $N'_H = N'_G$, then
$\alpha N_H = \beta N_G$. Therefore $\frac{P_H}{N_H} =
\frac{P_G}{N_G}$.

\subsubsection{Property used in the proofs for set operations}
Let $\Gamma_1$ be $\frac{lcm(P_{G_1}, P_{G_2})N_{G_1}}{P_{G_1}}$ and
$\Gamma_2$ be $\frac{lcm(P_{G_1}, P_{G_2})N_{G_2}}{P_{G_2}}$. Since
$G_1$ and $G_2$ are aligned subgranularity of a certain granularity
$H$, from Proposition~\ref{prop:SetOperations1} we can easily derive
that $\Gamma_1 = \Gamma_2$.

\subsection{Proof of Proposition~\ref{pr:setOperation}}
\subsubsection{Part 1}
\smallskip\noindent \textbf{Union}.
Let $\Gamma_1$ be $\frac{lcm(P_{G_1}, P_{G_2})N_{G_1}}{P_{G_1}}$ and
$\Gamma_2$ be $\frac{lcm(P_{G_1}, P_{G_2})N_{G_2}}{P_{G_2}}$.
The thesis will be proved by showing that $\forall i \in \Lset{G'}$ if,
$G'(i)=\bigcup_{r=0}^{k} \bot(i_r)$, then $G'(i+\Delta)=
\bigcup_{r=0}^{k}\bot(i_r + lcm(P_{G_1}, P_{G_2}))$ with $\Delta =
\Gamma_1 = \Gamma_2$. Since ${\cal L}_{G'} = {\cal L}_{G_1} \cup {\cal
  L}_{G_2}$, two cases will be considered:
\begin{itemize}
\item $\forall i \in {\cal L}_{G_1} \ G'(i)=G_1(i)=\bigcup_{r=0}^{k}
  \bot(i_r)$. {}From the periodicity of $G_1$, $G_1(i+\Gamma_1) =
  \bigcup_{r=0}^{k}\bot(i_r + lcm(P_{G_1}, P_{G_2}))$; hence
  $G'(i+\Gamma_1) = \bigcup_{r=0}^{k}\bot(i_r + lcm(P_{G_1},
  P_{G_2}))$.
\item $\forall i \in {\cal L}_{G_2} - {\cal L}_{G_1} \ 
  G'(i)=G_2(i)=\bigcup_{r=0}^{k} \bot(i_r)$. {}From the periodicity of
  $G_2$, $G_2(i+\Gamma_2) = \bigcup_{r=0}^{k}\bot(i_r + lcm(P_{G_1},
  P_{G_2}))$; hence $G'(i+\Gamma_2) = \bigcup_{r=0}^{k}\bot(i_r +
  lcm(P_{G_1}, P_{G_2}))$.

\end{itemize}

Since $\Gamma_1=\Gamma_2$, then $\forall i \in {\cal L}_{G'}$ if
$G'(i)=\bigcup_{r=0}^{k} \bot(i_r)$, then $G'(i+\Gamma_1) =
G'(i+\Gamma_2) = \bigcup_{r=0}^{k}\bot(i_r + lcm(P_{G_1},
P_{G_2}))$. Hence, by definition of $\gpinto$, we have the thesis.

\smallskip\noindent \textbf{Intersect}.
$\forall i \in {\cal L}_{G'}={\cal L}_{G_1} \cap {\cal L}_{G_2} \;
G'(i)=G_1(i) = \bigcup_{r=0}^{k}\bot(i_r)$. {}From the periodicity of
$G_1$ and $G_2$, $i+\Gamma_1 \in {\cal L}_{G_1}$ e $i+\Gamma_2 \in
{\cal L}_{G_2}$; since $\Gamma_1 = \Gamma_2$, then $i+\Gamma_1 \in
{\cal L}_{G'}$.  Moreover
$G'(i+\Gamma_1)=G_1(i+\Gamma_1)=\bigcup_{r=0}^{k}\bot(i_r+lcm(P_{G_1},
P_{G_2}))$; hence, by the definition of $\gpinto$, we have the thesis.

\smallskip\noindent \textbf{Difference}.
$\forall i \in {\cal L}_{G'}={\cal L}_{G_1}-{\cal L}_{G_2} \;
G'(i)=G_1(i) = \bigcup_{r=0}^{k}\bot(i_r)$. Since $i \in {\cal
  L}_{G_1}$ from the periodicity of $G_1$ $i+\Gamma_1 \in {\cal
  L}_{G_1}$. Since $i \notin {\cal L}_{G_2}$, from the periodicity of
$G_2$, $i+\Gamma_2 \notin {\cal L}_{G_2}$ (if it would exists
$i+\Gamma_2 \in {\cal L}_{G_2}$, from periodicity of $G_2$ would
exists $i \in {\cal L}_{G_2}$ that is not possible for hypothesis).
Hence $i+\Gamma_1 \in {\cal L}_{G'}$.  Moreover
$G'(i+\Gamma_1)=G_1(i+\Gamma_1)=\bigcup_{r=0}^{k}\bot(i_r+lcm(P_{G_1},
P_{G_2}))$; hence, by the definition of $\gpinto$, we have the thesis.

\subsubsection{Part 2}
Let $\Ltil{G'}= \Lcap{G_1}^{P_{G'}} \cup \Lcap{G_2}^{P_{G'}}$.

We show that $\Ltil{G'} = \Lcap{G'}$ by proving that 
(1) $\Ltil{G'} \subseteq \Lcap{G'}$ and
(2) $\Ltil{G'} \supseteq \Lcap{G'}$.

(1) Suppose by contradiction that $\exists k \in \Ltil{G'} \setminus \Lcap{G'}$. Since $k \in \Ltil{G'}$ then $k \in \Lcap{G_1}^{P_{G'}}$ or $k \in \Lcap{G_2}^{P_{G'}}$. Suppose that $k \in \Lcap{G_1}^{P_{G'}}$ (the proof is analogous if $k \in \Lcap{G_2}^{P_{G'}}$). Since $k \in \Lcap{G_1}^{P_{G'}}$, then $\exists \; 0<h<P_{G'}$ s.t. $\up{h}^{G'} = k$. Since, by definition of the \textsl{Union} operation $G'(k) = G_1(k)$, then $\up{h}^{G'} = k$. Hence, by definition of $\Lcap{}$, $k \in \Lcap{G'}$ in contrast with the hypothesis.

(2) Suppose by contradiction that $\exists k \in \Lcap{G'} \setminus \Ltil{G'}$. Since $k \in \Lcap{G'}$, then, by definition of $\Lcap{}$, $\exists \; 0<h<P_{G'}$ s.t. $\up{h}^{G'}=k$. Moreover, by definition of the \textsl{Union} operation, $k \in \Lset{G_1}$ or $k \in \Lset{G_2}$. Suppose that $k \in \Lcap{G_1}^{P_{G'}}$ (the proof is analogous if $k \in \Lcap{G_2}^{P_{G'}}$). By definition of the \textsl{Union} operation, $G'(k) = G_1(k)$ therefore $\up{h}^{G_1} = k$ and so, by definition of $\Lcap{}$, $k \in \Lcap{G_1}^{P_{G'}}$. Hence, by definition of $\Ltil{}$, $k \in \Ltil{G'}$ in contradiction with the hypothesis.

\bibliography{bib}
\bibliographystyle{theapa}

\end{document}